\documentclass[11pt]{article}
\usepackage{multirow}
\usepackage{amsmath,amssymb}
\usepackage{relsize}
\usepackage{graphicx}
\usepackage{listings}
\usepackage{xcolor}
\usepackage{booktabs}
\usepackage{fancyref}
\usepackage[parfill]{parskip}
\usepackage{todonotes}
\usepackage{import}
\usepackage{dsfont}
\usepackage{lineno}
\usepackage{float}
\usepackage{caption}
\usepackage[skip=10pt]{subcaption}
\usepackage{bm}
\usepackage{bbm}
\usepackage{commath}
\usepackage[section]{placeins}
\usepackage{fancyhdr}
\usepackage{natbib}
\usepackage{datetime}
\usepackage[hidelinks]{hyperref}

\def\RR{\mathbb{R}}

\numberwithin{equation}{section}

\hypersetup{pdfinfo={
	CreationDate={D:20200407195600},
	ModDate={D:20200407195600}
}}

\usepackage{algorithm}
\usepackage{algorithmic}
\usepackage{./docstyle}
\usepackage{enumitem}
\usetikzlibrary{calc}
\usetikzlibrary{bayesnet}
\usepackage{pythontex} 
\usepackage[draft]{minted}

\usepackage{wrapfig} 
\usepackage{color} 
\definecolor{darkgreen}{RGB}{0,179,0}

\newcommand{\myref}[2]{\hyperref[#2]{#1 \ref*{#2}}}

\newboolean{neurips}
\setboolean{neurips}{false} 
\newcommand{\na}{\textcolor{gray}{\footnotesize N/A}}
\newcommand{\tnan}{---}

\usepackage{amsthm}
\newtheorem{theorem}{Theorem}[section]
\newtheorem{lemma}{Lemma}[section]
\theoremstyle{definition} \newtheorem{definition}{Definition}[section]

\newdate{date}{26}{05}{2021}

\begin{document}

\title{ELF: Exact-Lipschitz Based Universal Density Approximator Flow}

\author{%
  Achintya Gopal \\
  Bloomberg Quant Research\\
  % 731 Lexington Avenue \\
  % New York, NY 10022 \\
    \texttt{achintyagopal@gmail.com} \\
}

\date{\displaydate{date}}
\maketitle

% \newcommand\blfootnote[1]{%
%   \begingroup
%   \renewcommand\thefootnote{}\footnotetext{#1}%
%   \addtocounter{footnote}{0}%
%   \endgroup
% }
% \blfootnote{Correspondence to \texttt{achintyagopal@gmail.com}}

\begin{abstract}
Normalizing flows have grown more popular over the last few years; however, they continue to be computationally expensive, making them difficult to be accepted into the broader machine learning community. In this paper, we introduce a simple one-dimensional one-layer network that has closed form Lipschitz constants; using this, we introduce a new Exact-Lipschitz Flow (ELF) that combines the ease of sampling from residual flows with the strong performance of autoregressive flows. Further, we show that ELF is provably a universal density approximator, more computationally and parameter efficient compared to a multitude of other flows, and achieves state-of-the-art performance on multiple large-scale datasets.
\end{abstract}

\section{Introduction}

Normalizing flows have become more popular within the last few years; however, they continue to have limitations compared to other generative models, more specifically that they are computationally expensive in terms of memory and time.

Early implementations of Normalizing Flows were coupling layers \citep{dinh2014nice,Dinh2016NVP, Kingma2018Glow} and autoregressive flows \citep{MAF2017, IAF2016}. These have easy to compute log-likelihoods;
however, coupling layers tend to need quite a few parameters to achieve strong performance and autoregressive flows are extremely expensive to sample from.
The newer technique of residual flows \citep{Chen2019ResidualFlows} allows for models that are built on standard components and have inductive biases that favor simpler functions \citep{QuARFlows2020}; however, these have the problem of being expensive in terms of time for computing log-likelihoods and training, as well as require quite a few layers for strong performance.

Since the introduction of these models, there have been many developments that have lead to improvement in parameter efficiency such as FFJORD \citep{FFJORD2018}, a continuous normalizing flow, that has a dynamic number of layers. However, this too can have computational problems as having a few dynamics layers can lead to hundreds of implicit layers.

Among the flows introduced, the ones with provable universal approximation capability are Affine Coupling Layers \citep{dinh2014nice,Dinh2016NVP,CFINN2020}, Neural Autoregressive Flows (NAF, \cite{NAF2018}), Block NAFs (BNAF, \cite{BNAF2019}), Sum-of-Squares Polynomial Flow \citep{SOSFlow2019}, and Convex Potential Flows (CP-Flow, \cite{huang2021cpflow}). Though these have been shown to be universal approximators, they do not necessarily translate into faster, more efficient training, and some of the flows listed require the expensive sampling routine of autoregressive flows.

In this paper, we introduce a new universal density approximator flow, Exact-Lipschitz Flows (ELF), based on residual flows and autoregressive flows, retaining the strong performance and closed-form log-likelihoods of autoregressive flows with the efficiency of sampling from residual flows. We achieve this by introducing a simple class of one-dimensional one-layer networks with closed form Lipschitz constants from which we generalize to higher dimensions using hypernetworks \citep{ha2016hypernetworks}. Further, we show our flow achieves state of the art performance on many tabular and image datasets.

\section{Related Work}

\subsection{Flows}

\cite{Rezende2015NF} introduced planar flows and radial flows to use in the context of variational inference. Sylvester flows created by \citet{Sylvester2018} enhanced planar flows by increasing the capacity of each individual transform. Both planar flows and Sylvester flows can be seen as special cases of residual flows. The capacity of these models were studied by \citet{Zhifeng2019}.

\citet{MADE2015} introduced large autoregressive models which were then used in flows by Inverse Autoregressive Flows (IAF, \cite{IAF2016}) and Masked Autoregressive flows (MAF, \cite{MAF2017}). Larger autoregressive flows were created by Neural Autoregressive Flows (NAF, \cite{NAF2018}) and Block NAFs (BNAF, \cite{BNAF2019}), both of which have better log-likelihoods but cannot be sampled from efficiently.

Early state of the art performance on images with flows was achieved by coupling layers such as NICE \citep{dinh2014nice}, RealNVP \citep{Dinh2016NVP}, Glow \citep{Kingma2018Glow} and Flow++ \citep{Flow++2019}. Affine coupling layers have been shown to be universal density approximators in the case of a sufficient number of layers \citep{CFINN2020}.

The likelihood performance of images was further improved by variational dequantization \citep{Flow++2019}. Additional improvements were achieved by masked convolutions \citep{MaCow2019} and adding additional dimensions to the input \citep{ANF2020,VFlow2020}.

Whereas autoregressive methods and coupling layers have block structures in their Jacobian, Residual Flows \citep{Chen2019ResidualFlows} have a dense Jacobian; FFJORD \citep{FFJORD2018} is a continuous normalizing flow based on Neural ODEs \citep{NODE2018} and can be viewed as a continuous version of residual flows. \citet{zhuang2021mali} improved the training of FFJORD achieving state of the art performance among uniformly dequantized flows.

\subsection{Lipschitz Functions}

To the best of our knowledge, the first upper-bound on the Lipschitz constant of a neural network was described by \citet{Szegedy2014} as the product of the spectral norms of linear layers. Regularizing for 1-Lipschitz functions has been used in the context of GANs \citep{wgan2017,Gulrajani2017}; 1-Lipschitz was further enforced in GANs by using the power iteration method to approximate the Lipschitz constant of individual layers \citep{SNGAN2018}. Residual Flows \citep{Chen2019ResidualFlows} also used the power iteration method to ensure 1-Lipschitz in its residual blocks.

In general, computing the Lipschitz constant even for two layer neural networks is NP-hard \citep{Virmaux2018}. However, this theorem is for any arbitrary network, meaning any $n$ dimensional input, any $m$ dimensional output, and any non-linear activation. In this work, we focus on a specific subset of networks in which the Lipschitz constant can be computed efficiently. \cite{Anil2019} introduced a new activation function and methods to have $1$-Lipschitz neural networks of arbitrary depth. Further approaches to estimate the Lipschitz constant have used semidefinite programs \citep{Fazlyab2019Lipschitz}.

\section{Background}

Suppose that we wish to formulate a joint distribution on an $n$-dimensional real vector $x$. A flow-based approach treats $x$ as the result of a transformation $f^{-1}$ applied to an underlying vector $z$ sampled from a base distribution $p_z(z)$.
The generative process for flows is defined as:
\begin{align*}
   z & \sim p_z(z) 
\\ x &= f^{-1}(z) 
\end{align*}
where $p_z$ is often a Normal distribution and $f$ is an invertible function.  Using change of variables, the log likelihood of $x$ is
$$ \log p_x(x) = \log p_z\left (f(x) \right) + \log \abs{\text{det}\left(\frac{\partial f(x)}{\partial x}\right)} $$
To train flows (i.e. maximize the log likelihood of data points), we need to be able to compute the logarithm of the absolute value of the determinant of the Jacobian of $f$, also called the \textit{log-determinant}.
To construct large normalizing flows, we can compose smaller ones as this is still invertible and the log-determinant of this composition is the sum of the individual log-determinants.

Due to the required mathematical property of invertibility, multiple transformations can be composed, and the composition is guaranteed to be invertible.
Thus, in theory, a potentially complex transformation can be built up from a series a smaller, simpler transformations with tractable log-determinants.

Constructing a Normalizing Flow model in this way provides two obvious applications: drawing samples using the generative process and evaluating the probability density of the modeled distribution by computing $p_x(x)$. These require evaluating the inverse transformation $f$, the log-determinant, and the density $p_z(z)$. In practice, if either $f$ or $f^{-1}$ turns out to be inefficient, then one or the other of these two applications can become intractable. For evaluating the density, in particular, computing the log-determinant can be an additional trouble spot.
A determinant can be computed in $O(n^3)$ time for an arbitrary $n$-dimensional data space. However, in many applications of flows, such as images, $n$ is large, and a $O(n^3)$ cost per evaluation is simply too high to be useful. Therefore, in flow-based modeling, there are recurring themes of imposing constraints on the model that guarantee invertible transformations and log-determinants that can be computed efficiently.

\subsection{Autoregressive Flows}

For a multivariate distribution, the probability density of a data point can be computed using the chain rule: 
\begin{equation}
 p(x_1,\dots,x_n) = \prod_{i=1}^{n} p(x_i|x_{<i}) 
 \end{equation}
By using a conditional univariate normalizing flow $f_\theta(x_i|x_{<i})$ such as an affine transformation for each univariate density, we get an autoregressive flow \citep{MAF2017}. Given that its Jacobian is triangular, the determinant is easy to compute as it is the product of the diagonal of the Jacobian. 
These models have a tradeoff where the log-likelihood is parallelizable but the sampling process is sequential, or vice versa depending on parametrization \citep{IAF2016}. In this paper, due to the fact the Jacobian is triangular, we refer to $f_\theta$ as a \textit{triangular function}.
Given a naive implementation of sampling from autoregressive models, the number of times an autoregressive function would need to be called to sample would be $d$ times.

\subsection{Residual Flows}

A residual flow \citep{Chen2019ResidualFlows} is a residual network $\left(f(x) = x + g(x)\right)$ where the Lipschitz constant (\myref{Definition}{defn:lipschitz}) of $g$ is strictly less than one. This constraint on the Lipschitz constant ensures invertibility and is enforced using power iteration method to spectral normalize $g$; the transform is invertible using Banach's fixed point algorithm (\myref{Algorithm}{alg:fixed_point_iteration}) where the convergence rate is exponential in the number of iterations and is faster for smaller Lipschitz constants \citep{Behrmann2019}:
\begin{equation}
 \norm{x - x_n}_2 \leq \frac{\text{Lip}(g)^n}{1 - \text{Lip}(g)} \norm{x_1 - x_0}_2 
 \end{equation}

\begin{algorithm}[t]
   \caption{Inverse of Residual Flow via Fixed Point Iteration}
   \label{alg:fixed_point_iteration}
\begin{algorithmic}
   \STATE {\bfseries Input:} data $y$, residual block $g$, number of iterations $n$
   \STATE Initialize $x_0 = y$.
   \FOR{$i=1$ {\bfseries to} $n$}
   \STATE $x_i = y - g(x_{i-1})$
   \ENDFOR
\end{algorithmic}
\end{algorithm}

Evaluating the density given a residual flow is expensive as the log-determinant is computed by estimating the Taylor series:
\begin{equation}\label{eqn:resid_logdet}
 \ln\abs{J_f(x)} = \sum_{k=1}^{\infty} (-1)^{k+1} \frac{\text{tr}(J_g^k)}{k} 
\end{equation}
where the Skilling-Hutchinson estimator \citep{skilling1989eigenvalues} is used to estimate the trace and the Russian Roulette estimator \citep{RussianRoulette} is used to estimate the infinite series.

\section{ELF (Exact-Lipschitz Flows)}

One of the bottlenecks in the capacity of residual flows is, as shown by \citet{Behrmann2019}, that a single flow can transform the log-determinant up to:
$$ d \ln\left(1 - \text{Lip}(g) \right) \leq \ln\abs{\text{det}\ J_{f}(x)} \leq d \ln\left(1 + \text{Lip}(g) \right) $$
where $d$ is the dimensionality of the data, $g$ is the residual connection and $f = I + g$. Both the number of layers and Lipschitz constant of $g$ affect the expansion and contraction bounds of these flows; to make matters worse, since the spectral normalization of $g$ uses an upper bound of the Lipschitz constant that has been shown to be loose \citep{Virmaux2018}, the number of layers required to achieve a target expansion or contraction increases.
The main contribution of our residual flow is that the Lipschitz constant can be calculated efficiently and exactly. 

\subsection{Universal 1-D Lipschitz Function}\label{sec:one_d_lipschitz}

\begin{figure}[!bt]
\begin{subfigure}{0.45\linewidth}
\centering
\centerline{\includegraphics[width=\columnwidth]{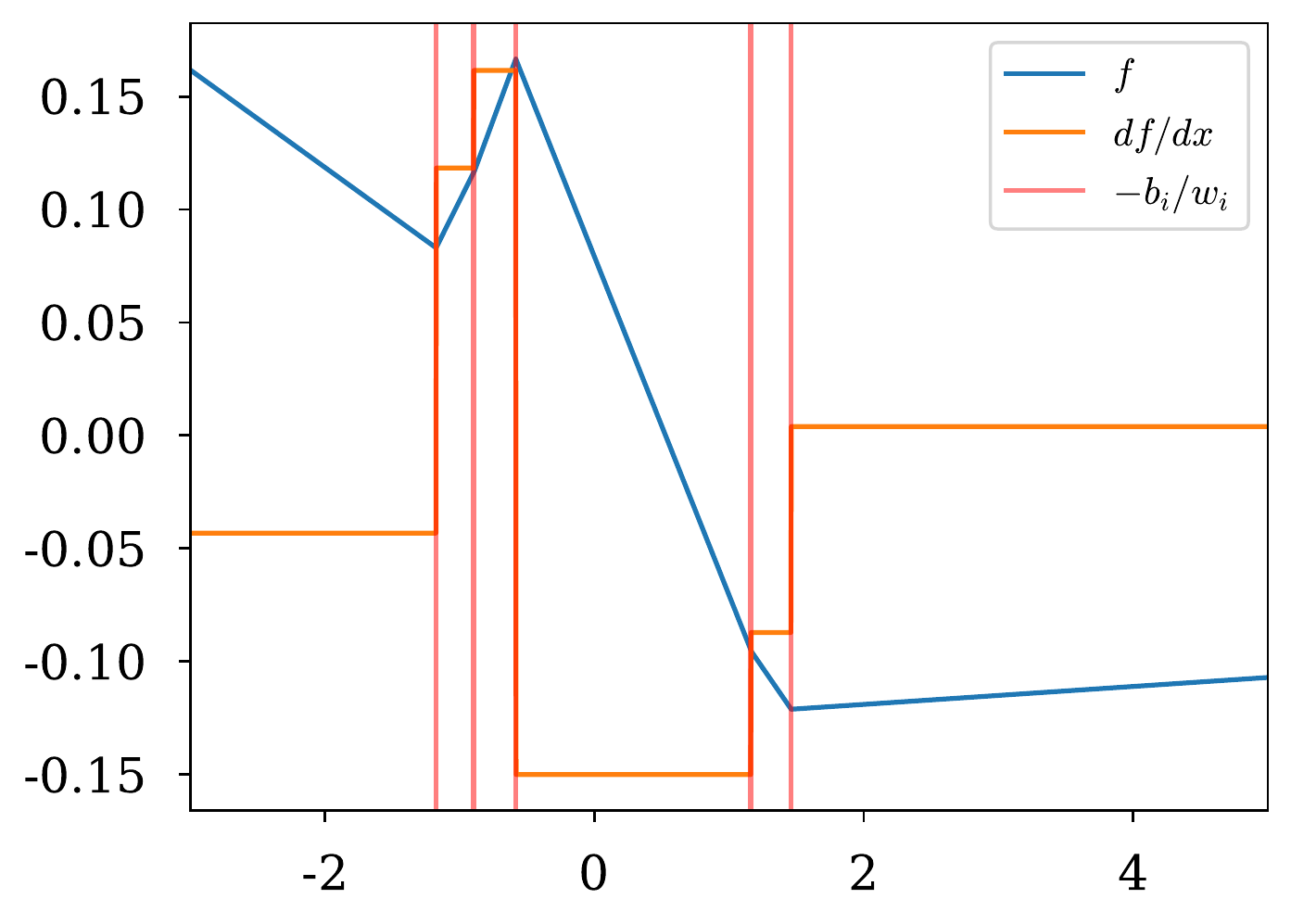}}
\caption{ReLU}
\label{fig:one_hidden_layer}
\end{subfigure}\hfill
\begin{subfigure}{0.45\linewidth}
\centering
\centerline{\includegraphics[width=\columnwidth]{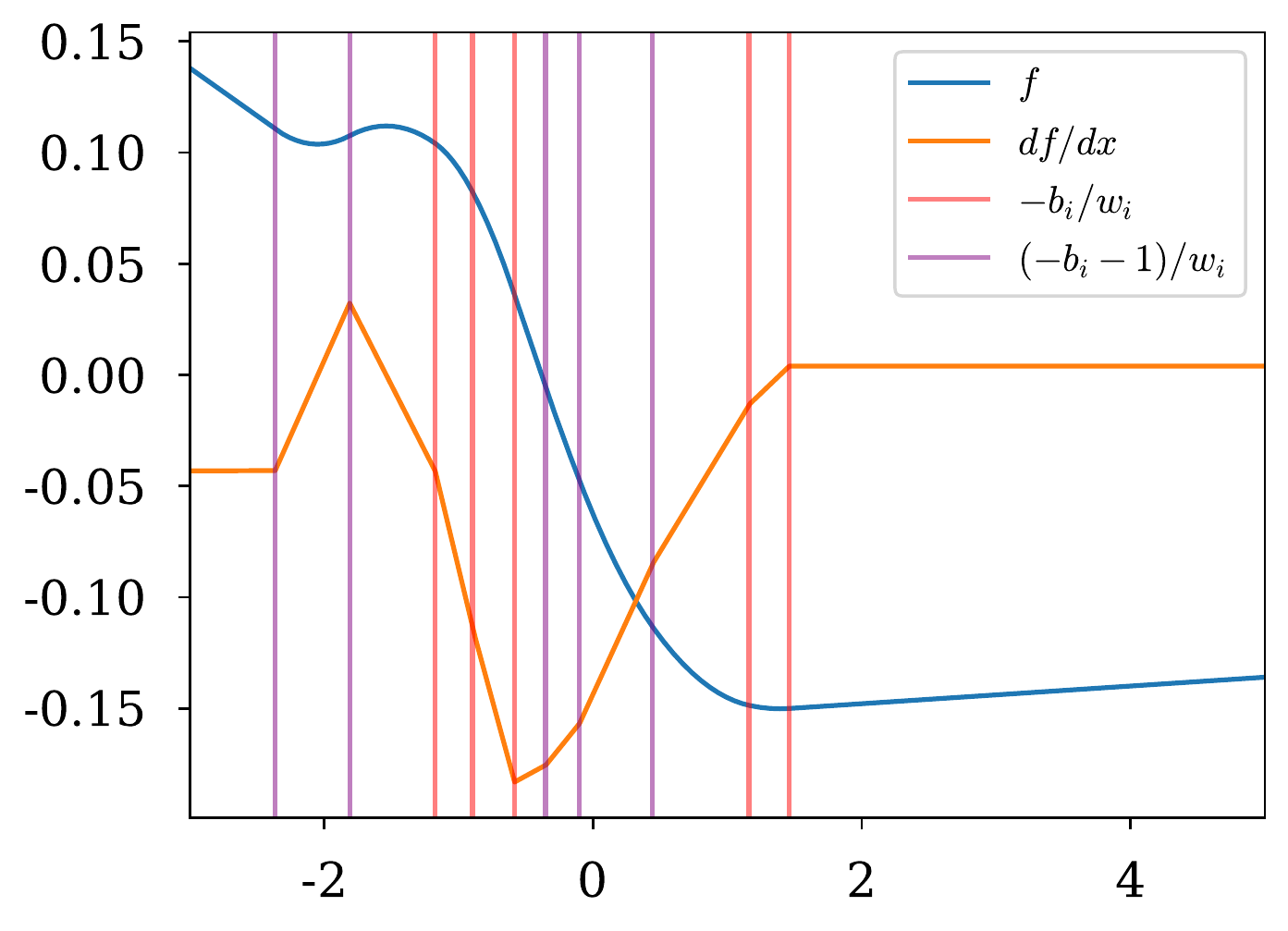}}
\caption{FELU}
\label{fig:one_hidden_layer_felu}
\end{subfigure}
\caption{Example of randomly initialized one layer networks with hidden size $H = 5$. Through the use of a quadratic piecewise activation function, computing the Lipschitz constant (the maximum absolute value of the gradient) requires simply evaluating the gradient at the ends of the pieces.}
\end{figure}

We start with the one dimensional case and generalize to higher dimensions in \myref{Section}{sec:elf_ar}. 
In \myref{Figure}{fig:one_hidden_layer}, we show an example of a randomly initialized one layer network with ReLU, both the function and its first derivative. The key observation here is that for a network with hidden size $H$, there are $H+1$ different gradients and so the Lipschitz can be computed with $H + 1$ gradient evaluations.

\begin{wrapfigure}[16]{r}{0.4\textwidth}
\centering
\centerline{\includegraphics[width=0.4\textwidth]{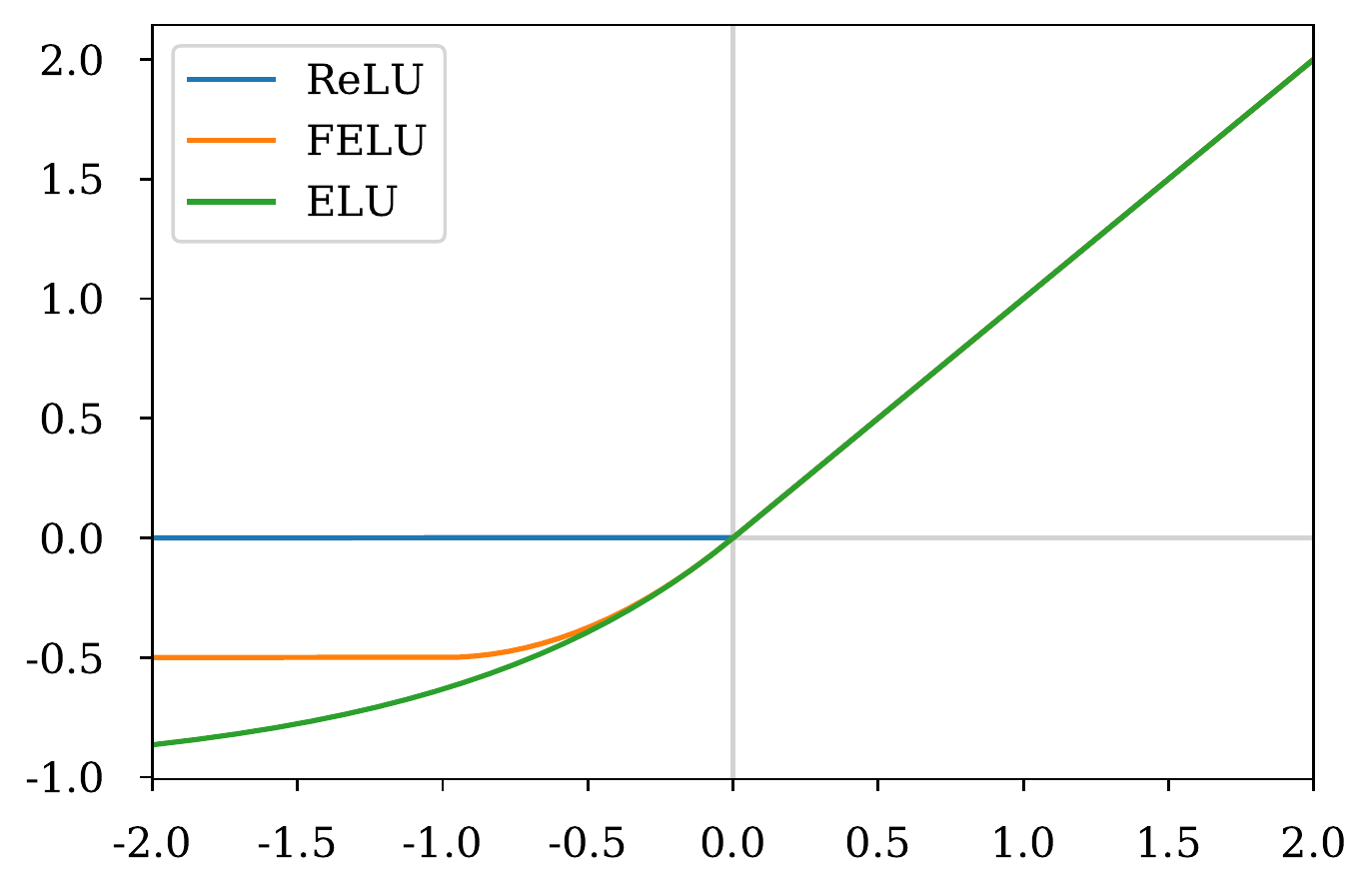}}
\caption{A comparison of ReLU, FELU, and ELU. The main shared components of FELU and ELU is the non-zero second derivative near zero and negative functions values for $x < 0$.}
\label{fig:felu}
\end{wrapfigure}
We can generalize this method of computing the Lipschitz constant by using any activation function that is piecewise linear and quadratic. Since ELU \citep{ELU2015} (a piecewise function with an exponential component) has been found to give strong performance in flows \citep{Chen2019ResidualFlows} and we were unable to successfully train a ReLU Residual flow (\myref{Appendix}{sec:relu_flow}), we created Fake ELU (FELU):
$$
\text{FELU}(x) = \begin{cases} 
  x & x > 0 \\
  \frac{(x + 1)^2}{2} - \frac{1}{2} & -1\leq x\leq 0 \\
  - \frac{1}{2} & x < -1
\end{cases}
$$
shown in \myref{Figure}{fig:felu}. \myref{Figure}{fig:one_hidden_layer_felu} shows the same network as in \myref{Figure}{fig:one_hidden_layer} but with FELU. 
Though there are not a finite number of unique gradients as in ReLU, within each piecewise component, the maximum absolute gradient can be found by simply checking the two ends of the piece. 

Though not explored in this paper, other quadratic piecewise activation functions could be used where an activation function with $N$ pieces would require $((N - 1) \cdot H + 1)$ gradient evaluations to compute the Lipschitz constant. For the activation to be 1-Lipschitz, it needs to have linear pieces on the two ends.

We show that this one-dimensional architecture (a one-layer neural network with FELU) is a universal approximator of Lipschitz functions (\myref{Lemma}{lemma:felu_approx}). 

\subsection{Universal 1-D Distributions}\label{sec:one_d_dist_elf}

We can use the one-dimensional architecture from \myref{Section}{sec:one_d_lipschitz} to create a new one-dimensional distribution based on residual flows which we call Exact-Lipchitz Flows (ELF).

In residual flows ($f = I + g$), any neural network can be used for $g$ where, if we write $g$ as $g_L \circ \dots \circ g_1 $, an upper bound on $\text{Lip}(g)$ is 
\begin{equation}
 \text{Lip}(g) \leq \prod_{i=1}^{L} \text{Lip}(g_i)  
\end{equation}
and so to normalize $g$, each function $g_i$ is normalized independently.

For ELF, we use the following for $g(x)$:
\begin{equation}
 b_2 + \sum_{i=1}^{H} w_{2,i}\ \text{FELU} (w_{1, i} x + b_{1,i}) 
 \end{equation}
where $w_{1,i}, w_{2,i}, b_{1,i}, b_2, x \in \RR$, or, in other words, a one layer network with FELU as the activation function. Instead of normalizing each weight matrix independently, we normalize by the exact Lipschitz constant, as was shown in \myref{Section}{sec:one_d_lipschitz}. More specifically, we use the maximum of:
\begin{equation}\label{eqn:elf_lip}
 \max_{i \in [1,\dots,H]} \abs{\ \frac{\partial g}{\partial x}\Bigg|_{x=-\frac{b_{1,i}}{w_{1,i}}}} \qquad \text{ and } \max_{i \in [1,\dots,H]} \abs{\ \frac{\partial g}{\partial x}\Bigg|_{x=-\frac{1 + b_{1,i}}{w_{1,i}}}}
\end{equation}

By composing this variant of residual flows with an affine layer (\myref{Figure}{fig:arch}), we show in \myref{Lemma}{sec:universality_proof} that this is a universal density approximator of one dimensional distributions. The key idea of the proof is that since the composition of ELF and an affine transformation is a universal approximator of monotonic functions (\myref{Lemma}{lemma:t_to_lipschitz}), this implies a universal density approximator of one dimensional distributions (\myref{Lemma}{lemma:cfinn}).

\begin{figure}
\begin{center}
\scalebox{0.8}{
\begin{tikzpicture}[
  scale=0.9,
  every neuron/.style={
    circle,
    % draw,
    minimum size=0.3cm,
    thick
  },
  every data/.style={
    rectangle,
    % draw,
    minimum size=0.4cm,
    thick
  },
]

  \node [align=center,every neuron/.try, data 1/.try, minimum width=0.3cm] (input-x)  at ($ (2.5, 0) $) {\small$x$};

  \node [align=center,every neuron/.try, data 1/.try, minimum width=0.3cm, draw] (w3) at ($ (input-x) + (2.5, 0) $) {\small$h_3$};
  \node [align=center,every neuron/.try, data 1/.try, minimum width=0.3cm, draw] (w2) at ($ (w3) + (0, 1.0) $) {\small$h_2$};
  \node [align=center,every neuron/.try, data 1/.try, minimum width=0.3cm, draw] (w1) at ($ (w2) + (0, 1.0) $) {\small$h_1$};
  \node [align=center,every neuron/.try, data 1/.try, minimum width=0.3cm, draw] (w4) at ($ (w3) - (0, 1.0) $) {\small$h_4$};
  \node [align=center,every neuron/.try, data 1/.try, minimum width=0.3cm, draw] (w5) at ($ (w4) - (0, 1.0) $) {\small$h_5$};

  \node [align=center,every neuron/.try, data 1/.try, minimum width=0.3cm, draw] (x) at ($ (w3) + (2.5, 0) $) {\small$+$};

  \node [align=center,every neuron/.try, data 1/.try, minimum width=0.3cm] (aff_t) at ($ (x) + (2.0, 0) $) {\small{ActNorm}};

  \node [align=center,every neuron/.try, data 1/.try, minimum width=0.3cm] (x2) at ($ (aff_t) + (1.75, 0) $) {\small$z$};

  \draw [black,solid,->] ($(input-x.east)$) -- ($(w1.west)$);
  \draw [black,solid,->] ($(input-x.east)$) -- ($(w2.west)$);
  \draw [black,solid,->] ($(input-x.east)$) -- ($(w3.west)$);
  \draw [black,solid,->] ($(input-x.east)$) -- ($(w4.west)$);
  \draw [black,solid,->] ($(input-x.east)$) -- ($(w5.west)$);

  \draw [black,solid,->] ($(w1.east)$) -- ($(x.west)$);
  \draw [black,solid,->] ($(w2.east)$) -- ($(x.west)$);
  \draw [black,solid,->] ($(w3.east)$) -- ($(x.west)$);
  \draw [black,solid,->] ($(w4.east)$) -- ($(x.west)$);
  \draw [black,solid,->] ($(w5.east)$) -- ($(x.west)$);
  \draw [black,solid,->] ($(x.east)$) -- ($(aff_t.west)$);
  \draw [black,solid,->] ($(aff_t.east)$) -- ($(x2.west)$);

  \draw [black,solid,->] ($(input-x.east)$) to [out=-85,in=-95, looseness=2.5] ($(x.west)$);

  \node[fit={ ($(input-x.west |- w1.north) - (0.2,0)$) ($(x.east |- w5.south) + (0,-0.75) $)}, blue,thick,dotted, draw] (calib-rect) {} ;

  \node [align=center,blue] (label)  at ($ (w1)+(-2.5,0.1) $) {\small{ELF}};

\end{tikzpicture}
}
\end{center}
\caption{Diagram of our one dimensional universal density approximator flow. ActNorm denotes an elementwise affine transformation with data dependent initialization \citep{Kingma2018Glow}}\label{fig:arch}
\end{figure}
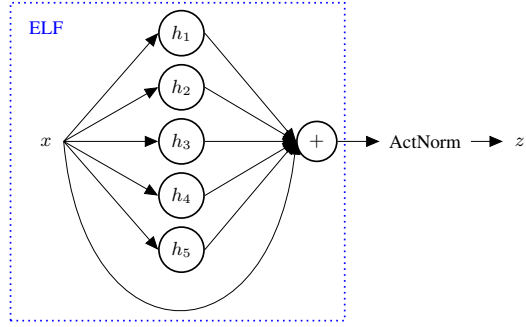

\subsection{Universal Approximation of Higher Dimensional Distributions}\label{sec:elf_ar}

To generalize to multivariate distributions, we use an autoregressive hypernetwork and refer to this as ELF-AR. 

Similar to autoregressive flows \citep{Oord2016PixelRNN,MAF2017,IAF2016,NAF2018,BNAF2019}, we use MADE \citep{MADE2015} to create large autoregressive networks that ensures the ordering is preserved. The output of the MADE is then the parameters of ELF. More specifically, we can write the autoregressive hypernetwork as:
\begin{equation}\label{eqn:elf_hnet}
 g(x)_t = b_2(x_{<t}) + \sum_{i=1}^{H} w_{2,i}(x_{<t})\ \text{FELU}\left( w_{1,i}(x_{<t})\cdot\ x_t + b_{1,i}(x_{<t})\right)
\end{equation}
where $x, g(x) \in \RR^{d}$, $t \leq d$, $H$ is the hidden size of ELF, and $b_2$,$w_{2,i}$, $w_{1,i}$, and $b_{1,i}$ are all hypernetworks that output a scalar. In terms of implementation, to maximize parallelism, we output all $3H +1 $ parameters in one call. Importantly, the hypernetwork is only over ELF, not the elementwise affine transformation in \myref{Figure}{fig:arch}. Once the hypernetwork outputs the parameters of ELF, each output dimension $g(x)_t$ is then normalized by the Lipschitz constant using \myref{Equation}{eqn:elf_lip}.

In \myref{Theorem}{thm:main_theorem}, we show that this architecture is a universal density approximator where the key idea of the proof is that the construction of one-dimensional 1-Lipschitz FELU Networks is continuous, hence we can use hypernetworks to create a universal approximator of triangular 1-Lipschitz functions. Thus creating a universal approximator of triangular functions by composing the ELF-AR with an elementwise affine transformation, \citet{CFINN2020} showed that this implies a universal density approximator.

In summary, the architecture introduced in this section uses a hypernetwork to parametrize ELF (\myref{Equation}{eqn:elf_hnet}). At which point, we compute the Lipschitz constant of each ELF (over every dimension) generated using \myref{Equation}{eqn:elf_lip} and normalize each ELF flow when the computed Lipschitz constant is greater than one. Complexity analysis and implementation details of \myref{Equation}{eqn:elf_lip} can be found in \myref{Appendix}{sec:complexity_analysis} and \myref{Appendix}{sec:cuda_code}. After normalizing by the Lipschitz constant, we can evaluate the function $f$ (\myref{Equation}{eqn:elf_hnet}) as well as its log-determinant in closed form:
\begin{equation}
\log \abs{\frac{f(x)}{dx}} = \log\abs{\sum_{t=1}^{d} \left[1 +  \frac{g(x)_t}{dx_t}\right]}
\end{equation}
Unlike Residual Flows, the log-determinant does not require estimating an infinite series (\myref{Equation}{eqn:resid_logdet}). By normalizing by the Lipschitz constant, inverting ELF (sampling) is efficient (\myref{Section}{sec:image_sample_eff}) since, instead of having to inverting each dimension sequentially as would be done for autoregressve flows, we apply the fixed point algorithm to the full input. 

\section{Experiments}

In this section, we test the approximation capabilities of ELF on simulated data (\myref{Section}{sec:lipschitz_approx}, \myref{Section}{sec:synthetic_data}), and performance on density estimation on tabular data (\myref{Section}{sec:tabular_data}) and image datasets (\myref{Section}{sec:image_data}).

\subsection{Approximation Power for Lipschitz Functions}\label{sec:lipschitz_approx}

\begin{wrapfigure}[16]{r}{0.45\textwidth}
\centering
\vspace{-1.5em}
\centerline{\includegraphics[width=0.45\textwidth]{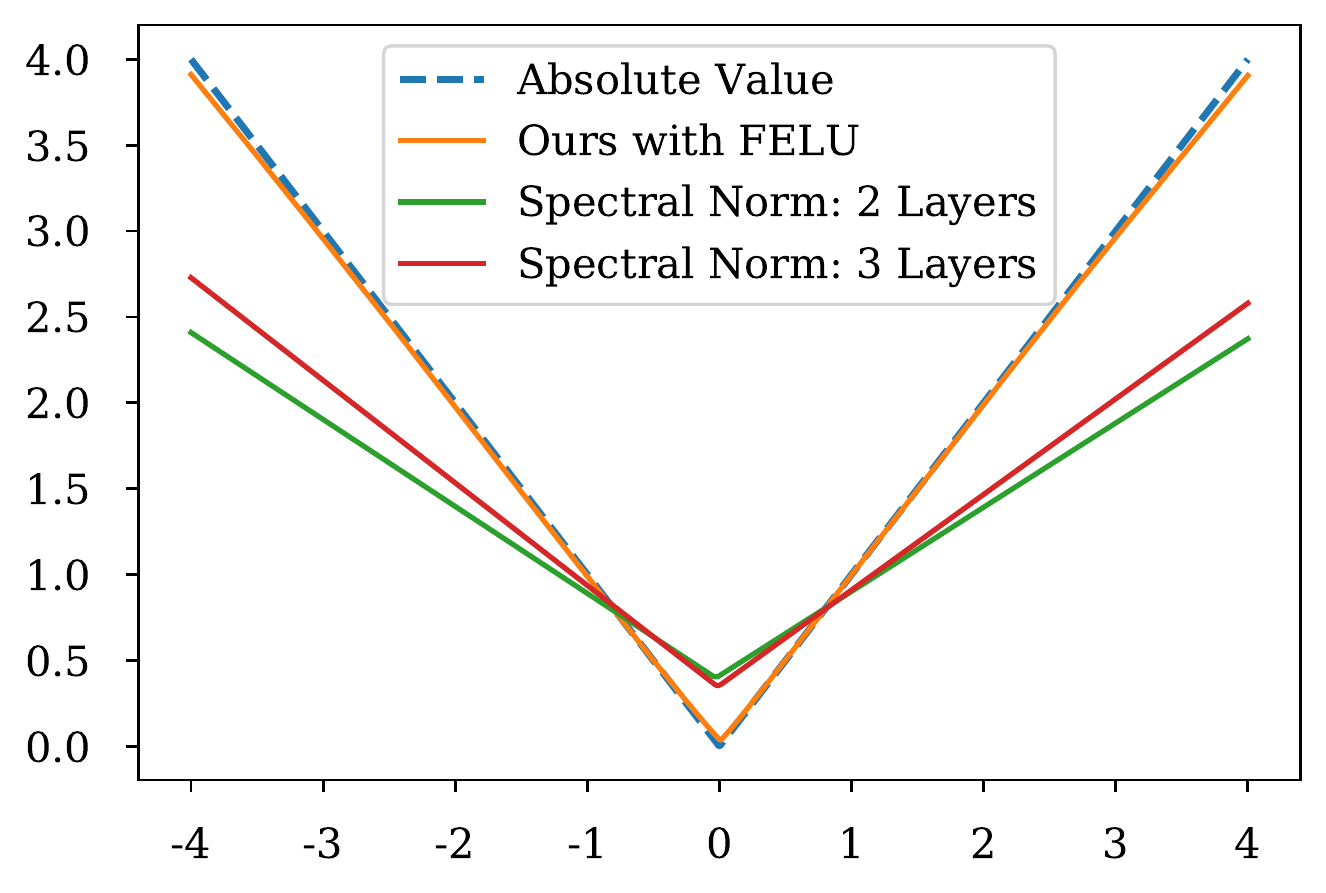}}
\caption{Comparison of using our closed form Lipschitz computation versus spectral normalization. Using closed form Lipschitz allows for absolute value to be learnable.}\label{fig:abs_val}
\end{wrapfigure}

For a simple use case comparing the power of exact Lipschitz computation versus using the product of the spectral norms of the linear layers, similar to \citet{Anil2019}, we check how well absolute value can be learned since, as was proven by \citet{huster2018limitations}, norm-constrained (such as spectral norm) ReLU networks are provably unable to approximate simple functions such as absolute value. In \myref{Figure}{fig:abs_val}, we can see that even if we use multiple layers to learn absolute value, spectral norm is too inaccurate to allow for learning absolute value, even though absolute value is a 1-Lipschitz function. On the other hand, by using an exact Lipschitz computation, our architecture is able to learn absolute value.

\subsection{Synthetic Data}\label{sec:synthetic_data}

Though ELF-AR with FELU is provably a universal approximator, this might not translate into improved performance in practice when training with gradient descent. To test this out, we compare the performance of one large ELF-AR against other universal density approximators. In \myref{Figure}{fig:synth_ring}, we can see, in comparison to autoregressive models, Glow, NAFs and CP Flows, ELF-AR is able to learn  a mixture of Gaussians quite well with only one transformation. However, ELF-AR's performance in sparse regions of the data suggests room for improvement in terms of extrapolation of the density function.

In \myref{Table}{tbl:synth_data}, we compare the runtime between the different flows. Though Affine Coupling is significantly faster than the other methods, it also has significantly worse performance. Further, we can see that ELF-AR, though a bit slower than the other flows in terms of sampling and inference, has stronger performance.

Performance on an additional synthetic dataset is in \myref{Appendix}{sec:synth_checkboard}. More details on the models trained is in \myref{Appendix}{sec:synth_architecture}.

\begin{figure}[!tb]
\begin{subfigure}{0.3\linewidth}
\centering
\centerline{\includegraphics[width=\columnwidth]{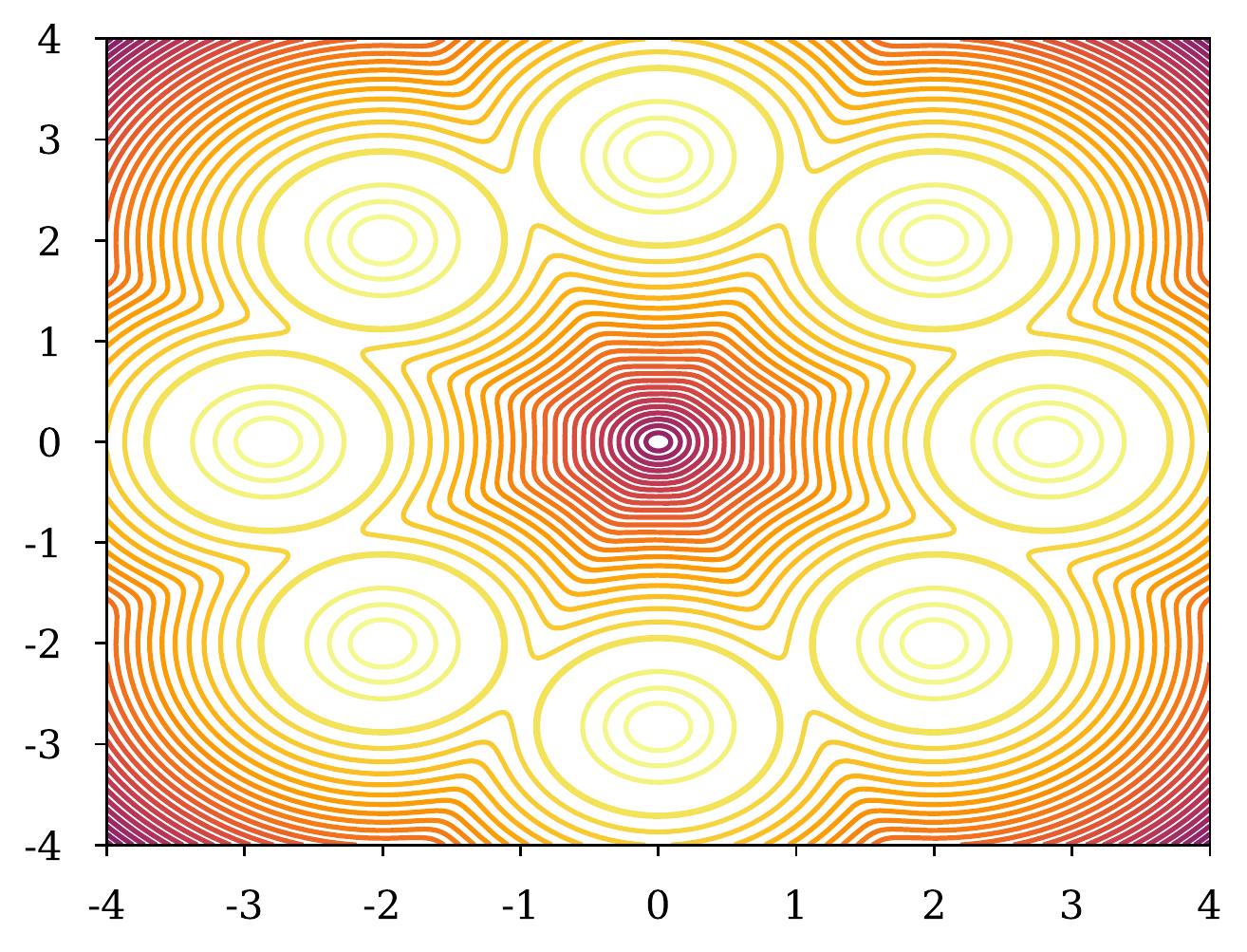}}
\caption{\textbf{True}}\label{fig:synth_true}
\end{subfigure} \hfill
\begin{subfigure}{0.3\linewidth}
\centering
\centerline{\includegraphics[width=\columnwidth]{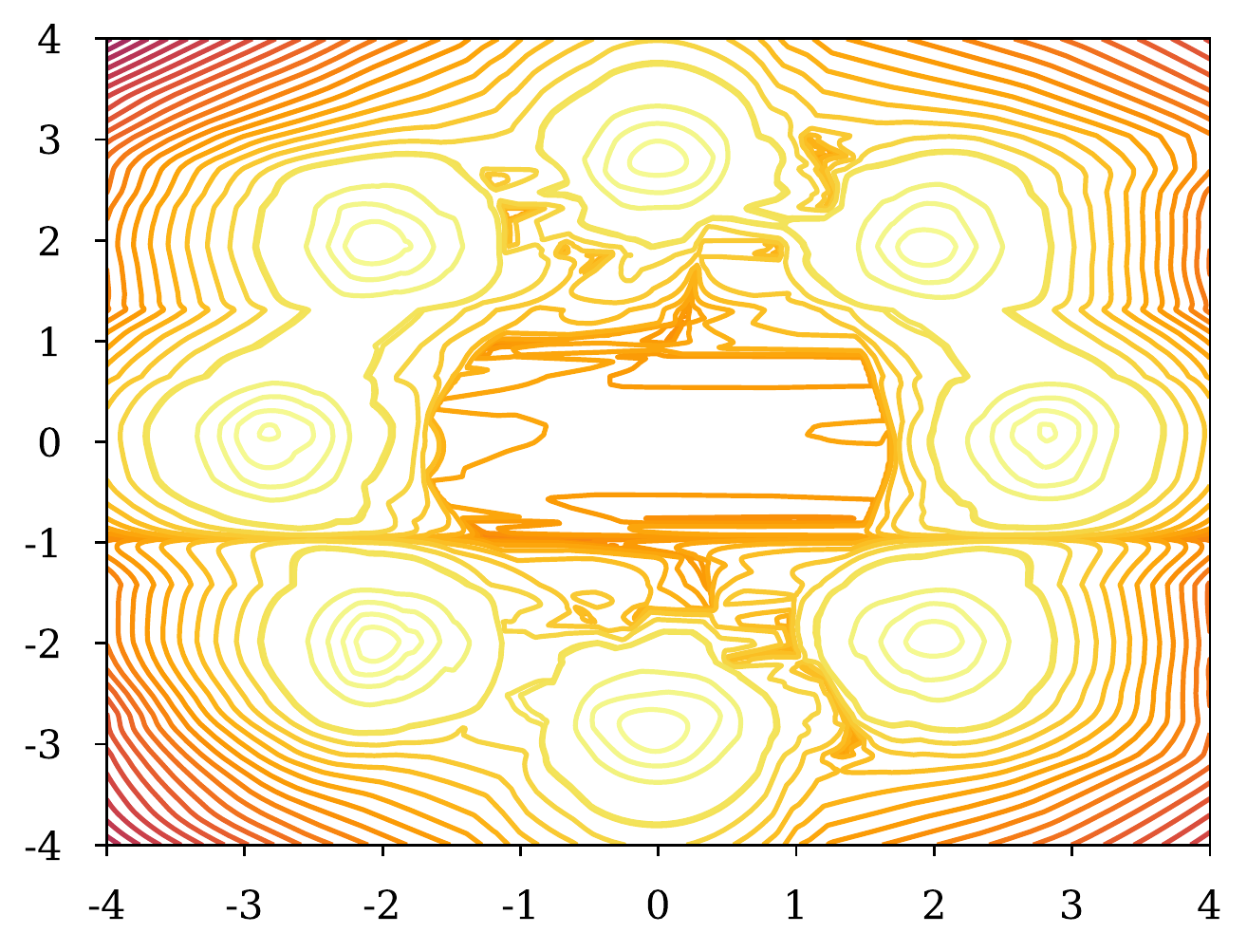}}
\caption{ELF-AR (Ours)}\label{fig:elf_ar_eightgauss_contour}
\end{subfigure} \hfill
\begin{subfigure}{0.3\linewidth}
\centering
\centerline{\includegraphics[width=\columnwidth]{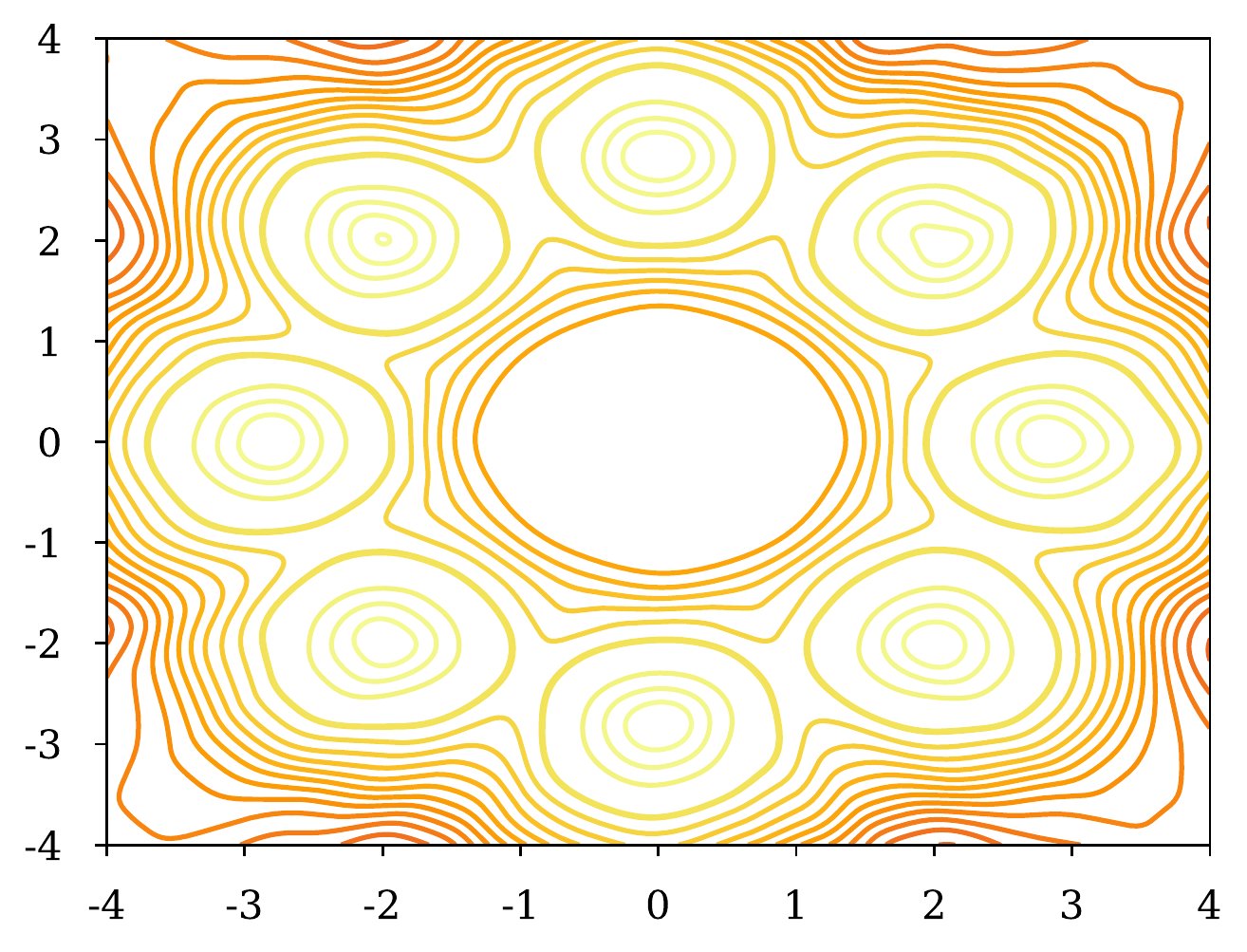}}
\caption{CP-Flow}
\end{subfigure}

\begin{subfigure}{0.3\linewidth}
\centering
\centerline{\includegraphics[width=\columnwidth]{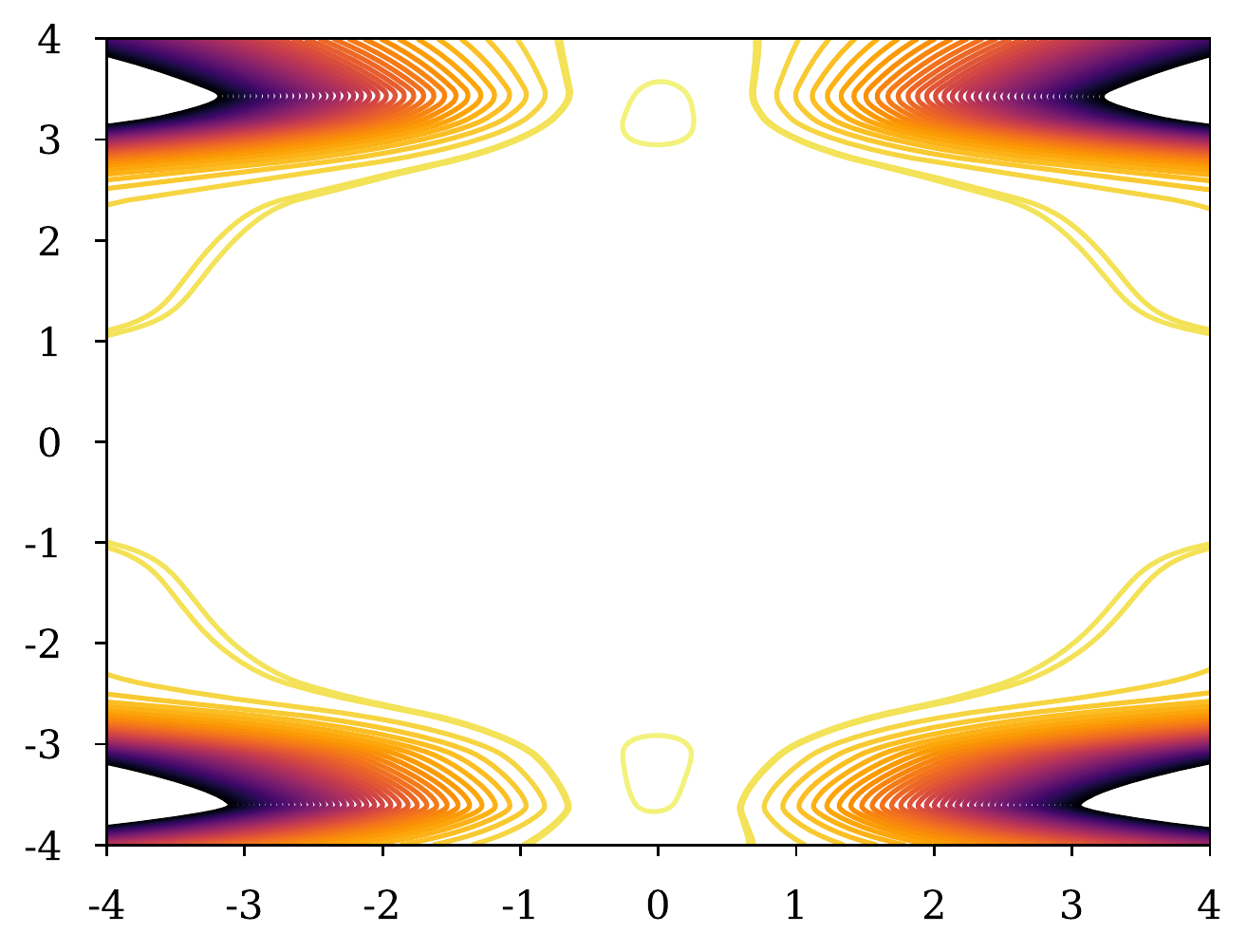}}
\caption{Glow (IAF)}
\end{subfigure} \hfill
\begin{subfigure}{0.3\linewidth}
\centering
\centerline{\includegraphics[width=\columnwidth]{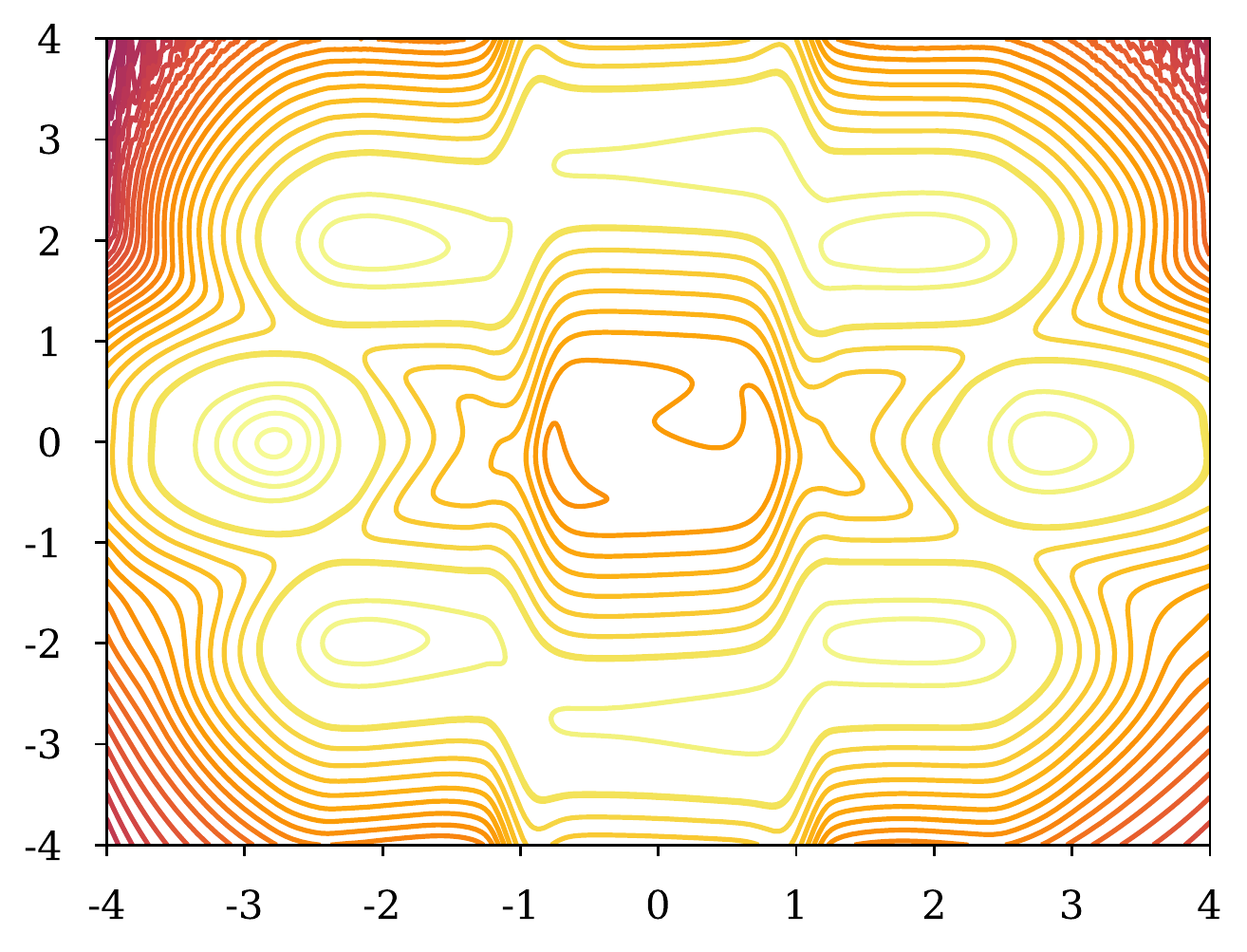}}
\caption{NAF}
\end{subfigure} \hfill
\begin{subfigure}{0.3\linewidth}
\centering
\centerline{\includegraphics[width=\columnwidth]{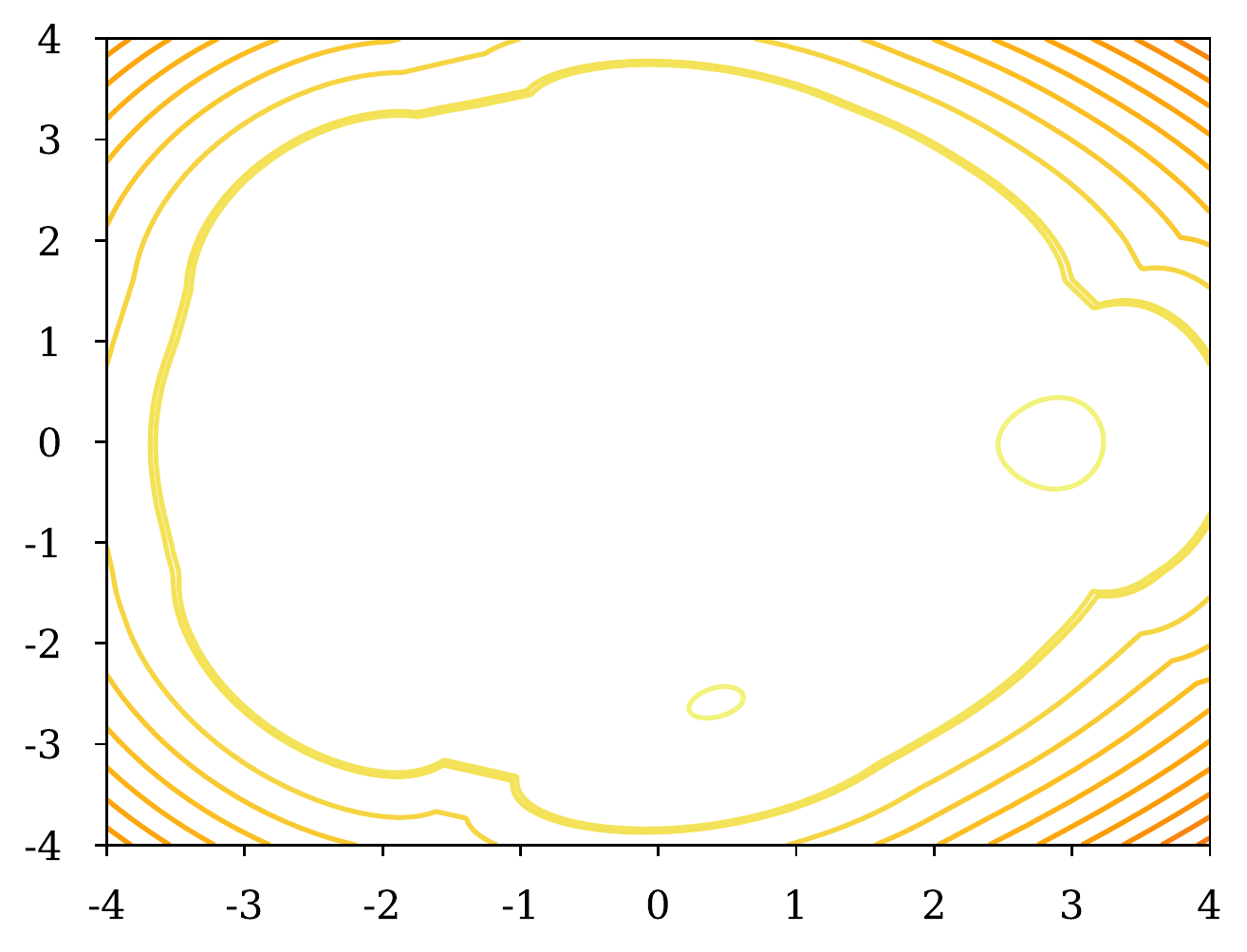}}
\caption{Residual Flow}
\end{subfigure}
\caption{The contour plots in log space of mixture of Gaussians. The levels shown in each subfigure are the same.}\label{fig:synth_ring}
\end{figure}

\begin{table*}[!tb]
\caption{Comparison of different flows on fitting to mixture of eight Gaussians (\myref{Figure}{fig:synth_true}) where the architectures for each were chosen so that they had the same number of layers and approximately equivalent number of parameters (250K). $^\dagger$Implementation taken from \href{https://github.com/CW-Huang/CP-Flow}{https://github.com/CW-Huang/CP-Flow} (MIT License).}\label{tbl:synth_data}
\centering
\begin{tabular}{lcccc}
\toprule
\multirow{1}{*}{\textbf{Model}}& {Log-Likelihood} & {Train Time (m)}& {Inference Time (s)} & {Sample Speed (s)}
 \\
\midrule
Affine Coupling & -4.0 & 1.8 & 0.005 & 0.003  \\ 
NAF$^\dagger$ & -3.0 & 15.5 & 0.02 & \na  \\ 
Residual Flow & -3.5 & 10.1 & 0.06 & 0.23  \\ 
CP-Flow$^\dagger$ & \textbf{-2.8} & 33.7 & 0.2 & 3.2  \\ 
\midrule
ELF-AR (Ours) & \textbf{-2.8} & 4.2 & 0.04 & 0.27  \\ 
\bottomrule
\end{tabular}
\end{table*}

\subsection{Tabular Data}\label{sec:tabular_data}

\begin{table*}[!tb]
\centering
\caption{Log-likelihood on the test set (higher is better) for $4$ datasets \citep{UCIDatasets} from UCI machine learning and BSDS300 \citep{BSDS3000}, all preprocessed by \citet{MAF2017}. We report average ($\pm \text{std}$) across $3$ independently trained models. }
\begin{tabular}{lccccc}
\toprule
\multirow{1}{*}{\textbf{Model}} & POWER & GAS & HEPMASS & MINIBOONE & BSDS300 \\
\midrule
Real NVP & $0.17${\tiny$\pm .01$} & $8.33${\tiny$\pm .14$} & $-18.71${\tiny$\pm .02$} & $-13.55${\tiny$\pm .49$} & $153.28${\tiny$\pm 1.78$} \\
Glow & $0.17${\tiny$\pm .01$} & $8.15${\tiny$\pm .40$} & $-18.92${\tiny$\pm .08$} & $-11.35${\tiny$\pm .07$} & $155.07${\tiny$\pm .03$} \\
MADE MoG & $0.40${\tiny$\pm .01$} & $8.47${\tiny$\pm .02$} & $-15.15${\tiny$\pm .02$} & $-12.27${\tiny$\pm .47$} & $153.71${\tiny$\pm .28$} \\
MAF-affine & $0.24${\tiny$\pm .01$} & $10.08${\tiny$\pm .02$} & $-17.73${\tiny$\pm .02$} & $-12.24${\tiny$\pm .45$} & $155.69${\tiny$\pm .28$} \\
MAF-affine MoG & $0.30${\tiny$\pm .01$} & $9.59${\tiny$\pm .02$} & $-17.39${\tiny$\pm .02$} & $-11.68${\tiny$\pm .44$} & $156.36${\tiny$\pm .28$} \\
FFJORD & $0.46${\tiny$\pm .01$} & $8.59${\tiny$\pm .12$} & $-14.92${\tiny$\pm .08$} & $-10.43${\tiny$\pm .04$} & $157.40${\tiny$\pm .19$} \\
CP-Flow & $0.52${\tiny$\pm .01$} & $10.36${\tiny$\pm .03$} & $-16.93${\tiny$\pm .08$} & $-10.58${\tiny$\pm .07$} & $154.99${\tiny$\pm .08$} \\
NAF & $\mathbf{0.62}${\tiny$\mathbf{\pm .01}$} & $11.96${\tiny$\pm .33$} & $-15.09${\tiny$\pm .40$} & $\mathbf{-8.86}${\tiny$\mathbf{\pm .15}$} & $157.43${\tiny$\pm .30$} \\
BNAF & $0.61${\tiny$\pm .01$} & $12.06${\tiny$\pm .09$} & $-14.71${\tiny$\pm .38$} & $-8.95${\tiny$\pm .07$} & $157.36${\tiny$\pm .03$} \\
TAN & $0.60${\tiny$\pm .01$} & $12.06${\tiny$\pm .02$} & $\mathbf{-13.78}${\tiny$\mathbf{\pm .02}$} & $-11.01${\tiny$\pm .48$} & $\mathbf{159.80}${\tiny$\mathbf{\pm .07}$} \\
\midrule
ELF-AR (Ours) & $\mathbf{0.62}${\tiny$\mathbf{\pm .01}$} & $\mathbf{12.55}${\tiny$\mathbf{\pm .03}$} & $-14.05${\tiny$\pm .03$} & $-9.60${\tiny$\pm .04$} & $157.81${\tiny$\pm .09$} \\
\bottomrule
\end{tabular}
\label{tab:tabular_density}
\end{table*}

For our tabular experiments, we use the datasets preprocessed by \citet{MAF2017}\footnote{CC by 4.0}. We compare ELF-AR against Real NVP \citep{Dinh2016NVP}, Glow \citep{Kingma2018Glow}, MADE \citep{MADE2015}, MAF \citep{MAF2017}, NAF \citep{NAF2018}, BNAF \citep{BNAF2019}, FFJORD \citep{FFJORD2018}, CP-Flow \citep{huang2021cpflow}, and TAN \citep{Oliva2018TAN}.

In \myref{Table}{tab:tabular_density}, we report average log-likelihoods evaluated on held-out test sets, for the best hyperparameters found via grid search. The search was focused on weight decay and the width of the hypernetwork and of ELF. In all datasets, ELF-AR performs better than Real NVP, Glow, MADE, MAF, FFJORD, and CP-Flow, and it performs comparably or better to NAF, BNAF, and TAN. 
More details on architecture and training details are in \myref{Appendix}{sec:tabular_architecture}.

\subsection{Image Data}\label{sec:image_data}

\begin{table*}[!t]
\caption{Bits-per-dim (\myref{Appendix}{sec:bpd}) estimates of standard image benchmarks (the lower the better). For variationally dequantized flows, we give the bits per dim for the uniformly dequantized versions in parenthesis.}\label{tbl:image_sota}
\centering
\begin{tabular}{lllll}
\toprule
\multirow{1}{*}{\textbf{Model}}  & {MNIST} & {CIFAR-10}& {Imagenet 32} & {Imagenet 64}
\\
\midrule
\textbf{Uniform Dequantization Flows} \\
RealNVP \citep{Dinh2016NVP}  & $1.06$ & $3.49$ & $4.28$ & $3.98$ \\ 
Glow \citep{Kingma2018Glow}  & $1.05$ & $3.35$ & $4.09$ & $3.81$ \\ 
FFJORD \citep{FFJORD2018} & $0.99$ & $3.40$ & \tnan & \tnan \\ 
MALI \citep{zhuang2021mali} & $0.87$ & $3.26$ & \tnan & \tnan \\ 
Residual Flow \citep{Chen2019ResidualFlows} & $0.97$ & $3.28$ & $4.01$ & $3.76$ \\ 
\midrule
\textbf{Variational Dequantization Flows} \\
Flow++ \citep{Flow++2019} & \tnan & $3.08 $\small{\ (3.29)}$ $ & $3.86$ & $3.69$ \\ 
MaCow \citep{MaCow2019} & \tnan & $3.16$ & \tnan & $3.69$ \\ 
ANF \citep{ANF2020} & $0.93$ & $3.05$ & $3.92$ & $3.66$ \\ 
VFlow \citep{VFlow2020} & \tnan & $\mathbf{2.98}$ & $\mathbf{3.83}$ & $3.66$ \\ 
\midrule
ELF-AR (Ours) & $\mathbf{0.85} $\small{\ (0.87)}$ $ & $3.06 $\small{\ (3.24)}$ $ & $3.92$ & $\mathbf{3.63}$ \\ 
\bottomrule
\end{tabular}
\end{table*}

We compare ELF-AR with variational dequantization against other state of the art flows. In \myref{Table}{tbl:image_sota}, we can see that we are able to achieve state of the art performance on MNIST \citep{lecun2010mnist} and Imagenet 64 \citep{chrabaszcz2017downsampled} while having competitive performance on the other image datasets. One possible direction of improvement to close the gap between ELF-AR and other state of the art flows on CIFAR-10 \citep{Krizhevsky09CIFAR10} and Imagenet 32 \citep{chrabaszcz2017downsampled} might be the use of attention \citep{vaswani2017attention}, as was done in Flow++ \citep{Flow++2019} and VFlow \citep{VFlow2020}.

More details on architecture and training details are in \myref{Appendix}{sec:architecture}.

\subsubsection{Sampling Efficiency}\label{sec:image_sample_eff}

The main benefit of ELF-AR over other autoregressive flows is the ability to use a more efficient sampling algorithm. For our best CIFAR-10 model, the average number of function evaluations required was 70; however, this count is not representative of the fact that different layers in the network take more function evaluations than others to invert. The runtime for sampling 100 images is 70s and the runtime, if we were to run the hypernetwork 3072 times (the dimensionality of CIFAR-10), would be 2600s. However, 3072 function evaluations is still not fully representative of the sampling time since the function introduced in \myref{Section}{sec:one_d_dist_elf} does not have a closed form inverse (similar to NAFs). Thus, the speedup of approximately 37 times is a lower bound of the true speedup. Though there are methods to use iterative methods for sampling from an autoregressive model \citep{song21aaccelerating,wiggers20arsampling}, these methods do not easily lend themselves to models where the univariate function (e.g. ELF) does not have a closed-form inverse.

\subsubsection{Parameter Efficiency}\label{sec:image_param_eff}

\begin{table*}[!tb]
\caption{Comparison of efficiency among uniform dequantized discrete flow-models specifically mentioning promoting fewer parameters. $^\dagger$Taken from~\citet{huang2021cpflow}. $^\ddagger$Obtained from official open source code.}\label{tbl:img_eff}
\centering
\begin{tabular}{lcccc}
\toprule
     & \multicolumn{2}{c}{MNIST}
     & \multicolumn{2}{c}{CIFAR-10}
\\
 \cmidrule(lr){2-3}
 \cmidrule(lr){4-5}
 & {Bits/dim} & {Param. Count}
 & {Bits/dim} & {Param. Count}
\\
\midrule
Glow \citep{Kingma2018Glow} & $1.06$  & \na & $3.35$ & 44.0M$^\dagger$  \\ 
RQ-NSF \citep{NSF2019} & \tnan  & \na & $3.38$ & 11.8M$^\dagger$  \\ 
Residual Flow \citep{Chen2019ResidualFlows} & $0.97$  & 16.6M$^\ddagger$ & $3.28$ & 25.2M$^\ddagger$  \\ 
CP-Flow \citep{huang2021cpflow} & $1.02$  & 2.9M$^\dagger$ & $3.40$ & 1.9M$^\dagger$  \\ 
\midrule
% Exact lipschitz ablation & $ $  & 1.9M & $ $ & 1.9M  \\ 
ELF-AR (Ours) & $\mathbf{0.92}$  & 1.9M & $\mathbf{3.33}$ & 1.9M  \\ 
% AR Residual Flow & -  & - & $3.27$ & 31.9M  \\ 
\bottomrule
\end{tabular}
\end{table*}

For evaluating the efficiency of ELF-AR, we measured the performance on MNIST and CIFAR-10 with approximately 2 million parameters without variational dequantization, with only 24 hours of training on a single Tesla V100-SXM2-32GB GPU. In \myref{Table}{tbl:img_eff}, we have the performance of ELF-AR alongside other models with their corresponding parameter counts.  With fewer parameters than any other model in the table, we were able to outperform all other flows except Residual Flows on CIFAR-10 and achieve state of the art among discrete flow models on MNIST with both a large margin in bits-per-dim and number of parameters.
Results for variationally dequantized models can be found in \myref{Appendix}{sec:vflow_comparison}.

\section{Limitations}\label{sec:limitations}

Though we were able to achieve strong performance, the method does seem to find bad local minimas when the data is highly multimodal. Specifically in the mixture of Eight Gaussians, when the clusters are smaller, there is a higher tendency to find bad solutions. We leave it to future work to find ways to improve this component which might help improve the overall performance in larger scale datasets.

Another limitation is that, due to the use of hypernetworks, modeling higher dimensional tabular data can lead to large parameter counts. Further, from \myref{Figure}{fig:elf_ar_eightgauss_contour}, it can be seen that ELF-AR does not seem to have the same inductive bias towards simple solutions as CP-Flow \citep{huang2021cpflow} and Residual Flows \citep{Chen2019ResidualFlows,QuARFlows2020}.

\section{Conclusion and Future Work}

In this work, we have introduced a new normalizing flow that allows for more efficient training, evaluation, and sampling compared to previous flows while achieving state of the art on multiple large-scale datasets. Specifically, we introduce a simple one-dimensional one-layer network that has closed form Lipschitz constants and introduce a new Exact-Lipschitz Flow (ELF) that combines the ease of sampling from residual flows with the strong performance of autoregressive flows. 

Though our model is provably a universal approximator, there is still a gap in performance between the state of the art flow for CIFAR-10 and our flow; future work entails exploring if using attention would further close the gap in performance.

As the flow we have introduced is a universal approximator that allows for more efficient sampling than autoregressive models, future research can further explore if the gap in performance of flows and autoregressive models such as PixelCNN++ \citep{PixelCNN++} is from dequantizating a discrete task, and what, if any, are the benefits of stacking autoregressive flows.

\section*{Acknowledgements}
We thank Wenjing Ge (Bloomberg Quant Research) for comments on the paper and reviewing the proofs.

\newpage
\bibliographystyle{chicago}
\bibliography{./biblio}
\newpage

\appendix

\section{Universality}\label{sec:universality_proof}

\begin{definition}\label{defn:lipschitz}
A function $f:\RR^d\rightarrow\RR^d$ is called \textit{Lipschitz continuous} if there exists a constant $L$, such that
$$ \norm{f(x_1) - f(x_2)}_2 \leq L\norm{x_1 - x_2}_2, \forall x_1, x_2 \in \RR^d $$
$f$ is called an \textit{$L$-Lipschitz function}. The smallest $L$ that satisfies the inequality is called the \textit{Lipschitz constant} of $f$, denoted as $\text{Lip}(f)$.
\end{definition}

\begin{definition}
We define $\mathcal{F}_{\sigma, H}$ as the family of functions $f:\RR^1\rightarrow\RR^1$ represented by a single layer neural network with a quadratic-piecewise 1-Lipschitz activation function $\sigma$ and hidden size $H$. $\mathcal{F}_{\text{FELU}, H}$ and $\mathcal{F}_{\text{ReLU}, H}$ are where the activation function is FELU and ReLU, respectively.
\end{definition}

\begin{definition}
We define $\mathcal{L}_{\sigma, H}$ as the family of functions where $g \in \mathcal{L}_{\sigma, H}$ if there exists a function $f \in \mathcal{F}_{\sigma, H}$ such that $g(x) = f(x) $ for all $x \in \RR$ if $\text{Lip}(f) \leq 1$ or $g(x) = f(x) / \text{Lip}(f) $ for all $x \in \RR$. 
\end{definition}

\begin{definition}($L^p$-/sup-universality \citep{CFINN2020})
Let $\mathcal{M}$ be a model which is a set of measurable mappings from $\RR^m\rightarrow\RR^n$. Let $p\in[1,\infty)$ and let $\mathcal{F}$ be a set of measurable mappings $f:U_f\rightarrow\RR^n$ where $U_f$ is a measurable subset of $\RR^m$ which may depend on $f$. We say that $\mathcal{M}$ is an \textit{$L^p$-approximator} for $\mathcal{F}$ if for any $f\in\mathcal{F}$, any $\epsilon > 0$, and any compact subset $K \subset U_f$, there exists a $g \in \mathcal{M}$ such that 
$$\left(\int_K \norm{f(x) - g(x)}_2^p \right)^{1/p} < \epsilon$$

We say that $\mathcal{M}$ is a \text{sup-approximator} for $\mathcal{F}$ if for any $f\in\mathcal{F}$, any $\epsilon > 0$, and any compact subset $K \subset U_f$, there exists a $g \in \mathcal{M}$ such that $\sup_{x \in K} \norm{f(x)-g(x)}_2 < \epsilon$.
\end{definition}

\begin{definition}(Elementwise affine transformation: $\mathcal{A}_d$)
We define $\mathcal{A}_d$ as the set of all elementwise affine transforms $\{x \rightarrow Ax + b | A \in \mathcal{D}, b \in \RR^d \}$ where $\mathcal{D}$ denotes the set of diagonal matrices $A\in\RR^{d\times d}$ with $A_{ii} \in \RR_{+}$.
\end{definition}

\begin{definition}(Triangular transformation: $\mathcal{T}^{\infty}_d$ \citep{CFINN2020})
We define $\mathcal{T}^{\infty}_d$ as the set of all $C^{\infty}$-increasing triangular mappings from $\RR^d \rightarrow \RR^d$. 
Here, a mapping $T=(T_1,\dots,T_d):\RR^d \rightarrow \RR^d$ is increasing triangular if for each $T_k(x) = T_k(x_k, x_{<k})$ is strictly increasing with respect to $x_k$ for a fixed $x_{<k}$ where $x \in \RR^d$.
\end{definition}

\begin{lemma}\label{lemma:cfinn}\citep[Lemma~1]{CFINN2020}
An $L^p$-universal approximator for $\mathcal{T}^{\infty}_d$ is a distributional universal approximator.
\end{lemma}

Using \myref{Lemma}{lemma:cfinn} and that sup-universality implies $L^p$-universality, we simply need to prove that a composition of an elementwise affine layer ($\mathcal{A}_d$) and ELF is a sup-universal approximator of $\mathcal{T}^{\infty}_d$.

\begin{definition}(1-Lipschitz triangular transformation: $\mathcal{P}^{\infty}_d$)
We define $\mathcal{P}^{\infty}_d$ as the set of all $C^{\infty}$-increasing 1-Lipschitz triangular mappings from $\RR^d \rightarrow \RR^d$.
Here, a mapping $P=(P_1,\dots,P_d):\RR^d \rightarrow \RR^d$ is increasing triangular if for each $P_k(x) = P_k(x_k, x_{<k})$ is strictly increasing with respect to $x_k$ for a fixed $x_{<k}$ where $x \in \RR^d$ and $P_k$ is 1-Lipschitz with respect to the first argument ($\sup_{x_k} \abs{\frac{\partial P_k}{\partial x_k}} \leq 1$).
\end{definition}

\begin{lemma}\label{lemma:t_to_lipschitz}
Given a set of triangular 1-Lipschitz functions $\mathcal{M}$ that is a sup-universal approximator of $\mathcal{P}^{\infty}_d$, then the set of functions $\{f \circ (I + g)\ \vert\ f \in \mathcal{A}_d, g \in \mathcal{M} \}$ is a sup-universal approximator of $\mathcal{T}^{\infty}_d$.
\end{lemma}

\begin{proof}

Given a multivariate continuously differentiable function $T(x)_t = T_{t}(x_t, x_{<t})$ for $t \in [1, m]$ that is strictly monotonic with respect to the first argument when the second argument in fixed (i.e. $T \in \mathcal{T}^{\infty}$), we define $D_t$ as:
$$ \sup_{x_{<t}} \abs{\frac{1}{2} \frac{\partial T_{t}(x_t, x_{<t})}{\partial x_t}} $$

If we divide $T(x)_t$ by $D_t$, then the resulting function is 2-Lipschitz. Further, we can write $T(x)_t = D_t(x_t + \left(\frac{T(x)_t}{D_t} - x_t)\right) = D_t\left(x_t + g(x_t, x_{< t})\right)$ where $g(x_t, x_{< t})$ is a 1-Lipschitz function with respect to the first argument. Say $\hat{g}(x_t, x_{< t})\in\mathcal{M}$ is $\epsilon / D_t$ close to $g(x_t, x_{< t})$, then:
$$ \sup_{x_t} \abs{D_t (x_t + \hat{g}(x_t, x_{< t})) - T(x)_t} \leq \epsilon $$
\end{proof}

From here, we simply need to prove that ELF is a universal approximator of triangular 1-Lipschitz functions.

\begin{lemma}\label{lemma:lip_eps}
Given a compact set $K \subset [a,b]$, if $f:\RR\rightarrow\RR$ and $g:\RR\rightarrow\RR$ are L-Lipschitz and are at least $\epsilon$ close at the endpoint $a$ and $b$, then

$$ \sup_{x\in[a,b]} \abs{f(x) - g(x)} \leq L (b - a) + \epsilon$$

\end{lemma}

\begin{proof}

\begin{align*}
   \sup_{x\in[a,b]} \abs{f(x) - g(x)} &\leq \sup_{x\in[a,b]} \min\left(\abs{g(x) - f(a)} + \abs{g(a) - g(x)}, \abs{g(x) - f(b)} + \abs{g(b) - g(x)} \right)
\\ &\leq \sup_{x\in[a,b]} \min\left(\abs{g(x) - g(a)} + \epsilon + \abs{g(a) - g(x)}, \abs{g(x) - g(b)} + \epsilon + \abs{g(b) - g(x)} \right)
\\ &\leq \sup_{x\in[a,b]} \min\left(2L \abs{x-a} + \epsilon, 2L \abs{x-b} + \epsilon \right)
\\ &\leq L (b - a) + \epsilon
\end{align*}

\end{proof}

\subsection{1-D Universality for ReLU}

\begin{definition}
We define $\mathcal{G}_{\text{ReLU}, H}$ as the set of functions $f \in \mathcal{F}_{\text{ReLU}, H}$ where the values of the weight matrix in the first layer are one.
\end{definition}

\begin{lemma}\label{lemma:relu_approx}
$\mathcal{G}_{\text{ReLU}, H}$ is a sup-approximator of $\mathcal{P}^{\infty}_1$.
\end{lemma}

\begin{proof}

Given some $\epsilon > 0$ and a compact set $K \subset [a,b]$, we choose $H = 2 \lceil{\frac{b - a}{\epsilon}}\rceil$ and divide the range $[a,b]$ into evenly spaced intervals $n=H/2$ intervals: $(x_0=a, x_1),(x_1, x_2),\dots, (x_{n-1}, x_{n}=b)$. We denote the target function as $f$ and the function we are learning:
$$ \hat{f}(x) = b_2 + \sum_{i=1}^{H} w_{i} \text{ReLU}(x + b_{1,i})  $$

For each interval $(x_i, x_{i+1})$, we set $b_{1,2i}$,$b_{1,2i+1}$, $w_{2i}$ and $w_{2i+1}$ such that $\hat{f}(x_i) = f(x_i)$ and $\hat{f}(x_{i+1}) = f(x_{i+1})$. 

Using induction over the intervals, we want $\hat{f} = f$ at each of the end points of the evenly spaced intervals and $\frac{\partial \hat{f}}{\partial x} = 0$ for $x < a$ and $\frac{\partial \hat{f}}{\partial x} = 0$ for $x > b$ (points to the right of the rightmost interval). We first set $b_2 = f(a)$ and all other parameters to zero. 

For the interval $(x_i, x_{i+1})$, we have $\hat{f}(x_i) = f(x_i)$. We set:
\begin{align*}
w_{2i} = \frac{f(x_{i+1}) - f(x_{i})}{x_{i+1} - x_{i}} &\qquad b_{1,2i} = -x_i
\\ w_{2i + 1} = -\frac{f(x_{i+1}) - f(x_{i})}{x_{i+1} - x_{i}} &\qquad b_{1,2i+1} = -x_{i+1}
\end{align*}

From the definition of ReLU, all points $x \leq x_i$ are unaffected by the change in parameters. Further, $\hat{f}(x_i) = f(x_i)$, $\hat{f}(x_{i+1}) = f(x_{i+1})$, and $\frac{\partial \hat{f}}{\partial x} = 0$ for $x > x_{i+1}$.

By construction and from the fact that $f$ is 1-Lipschitz, $\hat{f}$ is 1-Lipschitz. Using \myref{Lemma}{lemma:lip_eps}, within any interval:
$$ \sup_{x\in[x_i,x_{i+1}]} \abs{f(x) - \hat{f}(x)} \leq \epsilon$$

\end{proof}

\subsection{1-D Universality for FELU}

\begin{lemma}\label{lemma:felu_approx}
$\mathcal{F}_{\text{FELU}, H}$ is a sup-approximator of $\mathcal{P}^{\infty}_1$.
\end{lemma}

\begin{proof}
Given some $\epsilon > 0$ and a compact set $K \subset [a,b]$, we choose $H = 2 \lceil{2 \frac{b - a}{\epsilon}}\rceil$ and divide the range $[a,b]$ into evenly spaced intervals $n=H/2$
 intervals: $(x_0=a,x_1),(x_1,x_2),\dots, (x_{n-1}, x_{n}=b)$. We denote the target function as $f$ and the function we are learning:
% % 
\begin{equation}\label{eqn:tgt_felu_fn}
\hat{f}(x) = b_2 + \sum_{i=1}^{H} w_{2,i}\ \text{FELU}(w_{1,i}\ x + b_{1,i}) 
\end{equation}
% % 

Unlike for ReLU, we cannot set the parameters such that $\hat{f}(x_i) = f(x_i)$ and $\hat{f}(x_{i+1}) = f(x_{i+1})$ for all intervals. 

Instead, inductively over the intervals, we want to set the parameters such that if $\abs{\hat{f}(x_i) - f(x_i)} < \epsilon_1$, then $\abs{\hat{f}(x_{i+1}) - f(x_{i+1})} < \epsilon_1 + \frac{\epsilon}{2n}$ at each of the end points of the evenly spaced intervals. Further, we want $\frac{\partial \hat{f}}{\partial x} = 0$ for $x < a$ and $\frac{\partial \hat{f}}{\partial x} = 0$ for $x > b$ (points to the right of the rightmost interval). We first set $b_2 = f(a)$ and all other parameters to zero. 

The intuition of the following construction comes from:
$$
\lim_{w_1 \rightarrow \infty} w_2 / w_1 \text{FELU}(w_1 (x + b)) = w_2 \text{ReLU}(x + b)
$$

For the interval $(x_i, x_{i+1})$, we have $\abs{\hat{f}(x_i) - f(x_i)} < \epsilon_1$. We denote $L_i = \frac{f(x_{i+1}) - f(x_{i})}{x_{i+1} - x_{i}}$.
\begin{align*}
w_{1,2i} = &w_{1,2i+1} = 2n / \epsilon \\
w_{2,2i} = L_i / w_{1,2i}  &\qquad b_{1,2i} = -1 - x_i w_{1,2i} \\
w_{2,2i+1} =L_i / w_{1,2i+1}  &\qquad b_{1,2i+1} = - x_{i+1} w_{ 1,2i}
\\ b_2 = b_2 + (w_{2,2i} &+ w_{2,2i + 1}) / 2
\end{align*}
where we update the value for $b_2$ since $\text{FELU}(x) = -1/2$ for $x < -1$.

From construction, all values $x \leq x_i$ are unaffected and $\frac{\partial \hat{f}}{\partial x} = 0$ for $x > x_{i+1}$. Further, by construction, $\hat{f}$ is 1-Lipschitz and $\abs{\hat{f}(x_{i+1}) - f(x_{i+1})} < \epsilon_1 + 1 / w_{1,2i} = \epsilon_1 + \frac{\epsilon}{2n}$. 

Using \myref{Lemma}{lemma:lip_eps}:
\begin{align*}
\\ \sup_{x\in[a,b]} \abs{f(x) - \hat{f}(x)}  &\leq \sup_{i \in [0,\dots,n-1]} \sup_{x\in[x_i,x_{i+1}]} \abs{f(x) - \hat{f}(x)} 
\\ &\leq \sup_{i \in [0,\dots,n-1]} (i + 1) \frac{\epsilon}{2n} + \epsilon / 2
\\ &\leq n \frac{\epsilon}{2n} + \epsilon / 2
\\ &\leq \epsilon
\end{align*}

\end{proof}

\subsection{Universality of Higher Dimensional Distribution}

\begin{theorem}\label{thm:main_theorem}
Let $x \in [a,b]^{d}$ where $a, b \in \RR$. Given any $0 < \epsilon < 1$ and any multivariate continuously differentiable function $P(x)_t = P_{t}(x_t, x_{<t})$ for $t \in [1, d]$ that is strictly monotonic and 1-Lipschitz with respect to the first argument when the second argument in fixed (i.e. $P(x) \in \mathcal{P}^{\infty}_d$), then there exists a multivariate function $\hat{P} \in \mathcal{P}^{\infty}_d$, such that $\norm{\hat{P}(x) - P(x)} < \epsilon$ for all $x$, of the following form:
$$\hat{P}(x)_t = \frac{M_{x_{<t}}(x_t)}{\max (1, \text{Lip}(M_{x_{<t}}))}$$
where:
\begin{equation}\label{eqn:tgt_relu_fn}
 M_{x_{<t}}(x) = b_2 + \sum_{i=1}^{H} w_{i}\ \text{ReLU}(x + b_{1,i}) 
 \end{equation}
where $b_2$, $w_i$ and $b_1$ may depend on $x_{<t}$, and $\text{Lip}(M_{x_{<t}})$ is the Lipschitz constant of \myref{Equation}{eqn:tgt_relu_fn}.
\end{theorem}

\begin{proof}

We denote $C(x_{<t})$ as the function that maps from $x_{<t}$ to the parameters of \myref{Equation}{eqn:tgt_relu_fn} based on the construction in \myref{Lemma}{lemma:relu_approx} for the univariate function $P_t(\cdot, x_{<t})$.

For shorthand, we will use $c$ to denote $x_{<t}$,
use $M(c)$ to denote a neural (hyper)network that outputs the parameters of \myref{Equation}{eqn:tgt_relu_fn}, use $C_c$ and $M_c$ to denote the univariate function created by the outputs of $C(c)$ and $M(c)$ respectively, and use $L_M$ to denote $\max (1, \text{Lip}(M_c))$.

The crux of the proof is that from \myref{Lemma}{lemma:relu_approx}, we can construct an $\epsilon_1$-close approximation $C_c$ to $P_{t}(\cdot, x_{<t}):\RR\rightarrow\RR$ when the second argument is fixed, and we can then approximate $C(c)$ with a neural network to get an $\delta$-close approximation.
Since the specification in \myref{Lemma}{lemma:relu_approx} depends only on the function value at specific pre-defined points, the parameters of \myref{Equation}{eqn:tgt_relu_fn} (the output of the target function $C(c)$) is continuous with respect to $x_{<t}$. 
Using the fact that $C(c)$ is continuous, we can apply the classic results of \cite{cybenko1989approximation} which states that a multilayer perceptron can approximate any continuous function on a compact subset of $\RR^{d}$, giving us $M(c)$ as a $\delta$-close approximation to $C(c)$.

More specifically, given a fixed $c$:
\begin{align}
   \sup_{x\in[a,b]} \abs{M_c(x)/L_M - {P}(x)_t} &\leq  \sup_{x\in[a,b]}\abs{M_c(x)/L_M - C_c(x)} + \sup_{x\in[a,b]} \abs{ C_c(x) - {P}(x)_t}
\\ &\leq \sup_{x\in[a,b]} \abs{M_c(x)/L_M - C_c(x)} + \epsilon_1 \label{eqn:relu_proof_s1}
\end{align}

where $C_c$ is from \myref{Lemma}{lemma:relu_approx}.

Simplifying the first term:
\begin{align}
   \sup_{x\in[a,b]} \abs{M_c(x)/L_M - C_c(x)} &\leq \sup_{x\in[a,b]}  \abs{M_c(x) / L_M - M_c(x) } + \sup_{x\in[a,b]} \abs{M_c(x) - C_c(x) }
\\ &\leq \frac{\abs{L_M - 1}}{\abs{L_M}} \sup_{x\in[a,b]}  \abs{M_c(x)} + \sup_{x\in[a,b]}\abs{M_c(x) - C_c(x)}
\\ &\leq (L_M - 1) \sup_{x\in[a,b]}  \abs{M_c(x)} + \sup_{x\in[a,b]}\abs{M_c(x) - C_c(x)}\label{eqn:relu_proof_s2}
\end{align}

To bound the effect of having a $\delta$-close approximation to $C(c)$, we need to bound $L_M $ and $\sup_{x\in[a,b]} \abs{M_c(x)}$.

Since \myref{Equation}{eqn:tgt_relu_fn} and its Lipschitz constant are both continuous with respect to the parameters and on a compact set, the two are uniformly continuous. 
By uniform continuity, for some $\epsilon_2$, there exists a $\delta_1$ such that when $\norm{M(c) - C(c)}_2 < \delta_1$, $\abs{L_M - L_C} < \epsilon_2$. Thus,
\begin{align}
  L_M 
   &\leq \max(1, \text{Lip}(M_c))
\\ &\leq \max(1, \text{Lip}(C_c) + \epsilon_2)
\\ &\leq 1 + \epsilon_2
\end{align}

Similarly, for some $\epsilon_3$, there exists a $\delta_2$ such that when $\norm{M(c) - C(c)}_2 < \delta_2$, $\abs{M_c(x) - C_c(x)} < \epsilon_3$. Thus,
\begin{align}
   \sup_{x\in[a,b]} \abs{M_c(x)} &\leq \sup_{x\in[a,b]} \abs{M_c(x) - C_c(x)} + \sup_{x\in[a,b]} C_c(x) 
\\ &\leq \epsilon_3 + \norm{C_c}_{\infty}
\end{align}
where $\norm{C_c}_{\infty}$ is finite due to the compactness of the domain.

Using $\delta=\min(\delta_1, \delta_2)$ and that $M(c)$ is $\delta$-close $C(c)$, plugging all these into \myref{Equation}{eqn:relu_proof_s2},
\begin{align}
   \sup_{x\in[a,b]} \abs{M_c(x)/L_M - C_c(x)} &\leq  \epsilon_2 (\epsilon_3 + \norm{C_c}_{\infty}) + \epsilon_3 
\end{align}

Finally, using \myref{Equation}{eqn:relu_proof_s1}, $\epsilon_1 < \frac{\epsilon}{2}$, and $\epsilon_2 = \epsilon_3 < \min(\frac{\epsilon}{2}, \frac{1}{2} \frac{1}{2 + \norm{C_c}_{\infty}})$, we have
$$ \sup_{x,c} \abs{M_c(x)/L_M - {P}(x)_t}  < \epsilon $$

Having proved the univariate case, from the above, we know that given any $\epsilon$ > 0 for each $t$, there exists $\delta_t$ such that $\sup_{x\in[a,b]} \abs{P(x)_t - \hat{P}(x)_t} < \epsilon$ for all $x_{< t}$. Choosing $\delta_{\text{full}} = \max_{t\in[1,d]} \delta_t$,  we have $\norm{\hat{P}(x) - P(x)} < \epsilon$ for all $x$.

\end{proof}

Though we proved universality for ReLU, due to the construction of \myref{Lemma}{lemma:felu_approx}, the proof of universality is a straightforward generalization of the above proof since \myref{Equation}{eqn:tgt_felu_fn} and its Lipschitz constant are both continuous with respect to the parameters.

\section{Complexity Analysis of Lipschitz Constant Computation}\label{sec:complexity_analysis}

In \myref{Section}{sec:one_d_dist_elf}, we introduce \myref{Equation}{eqn:elf_lip} which shows the computation we perform to compute the Lipschitz constant of ELF. In this section, we perform a more thorough analysis of the runtime complexity of this computation. For optimization, we implemented the computation using CUDA to fully utilize the parallel capabilities of GPUs (\myref{Appendix}{sec:cuda_code}).

As was discussed in \myref{Section}{sec:one_d_lipschitz}, for a quadratic piecewise activation function with $N$ pieces for a one-layer network with hidden size $H$, the number of gradient evaluations required to compute the Lipschitz constant is $(N - 1) \cdot H + 1$. In the case of FELU (with three pieces), \myref{Equation}{eqn:elf_lip} requires only $2 \cdot H$ evaluations. The reason for the removal of the $+1$ is that the gradient is continuous for FELU and the two ends of the activation function is linear (i.e. constant gradient). 

If the gradient of the activation were not continuous (e.g. ReLU), evaluation at the points where the gradient changes (e.g. $w_{1,i} x_i + b_{1,i} = 0$ for ReLU) could still be used; however, a convention must be chosen for the gradient at that point. The choice of convention would inform which of the $N\cdot H+1$ gradients has not been calculated.

Focusing on FELU, we can see that the runtime complexity of $\frac{\partial g}{\partial x}$ is $O(H)$ (compute the gradient coming from each neuron). Further, since the number of gradient evaluations we need to do is $2H= O(H)$, the full runtime complexity is $O(H^2)$.

Empirically, we evaluate the runtime as a function of hidden size which we expect to be quadratic (\myref{Figure}{fig:runtime_hiddensize}) and as a function of batch size which we expect to be linear (\myref{Figure}{fig:runtime_batchsize}). Further, we verify the assumption by fitting a polynomial curve and checking the $R^2$ of the fit (0.998 and 0.993 respectively). The batch size in \myref{Figure}{fig:runtime_batchsize} indicates the number of one-dimensional Lipschitz constants computed.

To put into context the importance of the batch size, given the description of the model in \myref{Section}{sec:elf_ar}, the number of Lipschitz constants to be computed for a batch size $B$ for $D$-dimensional input would be $D\cdot B$ (this number can be thought of as the batch size in \myref{Figure}{fig:runtime_batchsize}). For example, for our MNIST model, for a batch size of 64, the number of Lipschitz constants computed per flow is $784\cdot 64=50176$.

\begin{figure}[!tb]
\begin{subfigure}{0.45\linewidth}
\centering
\centerline{\includegraphics[width=\columnwidth]{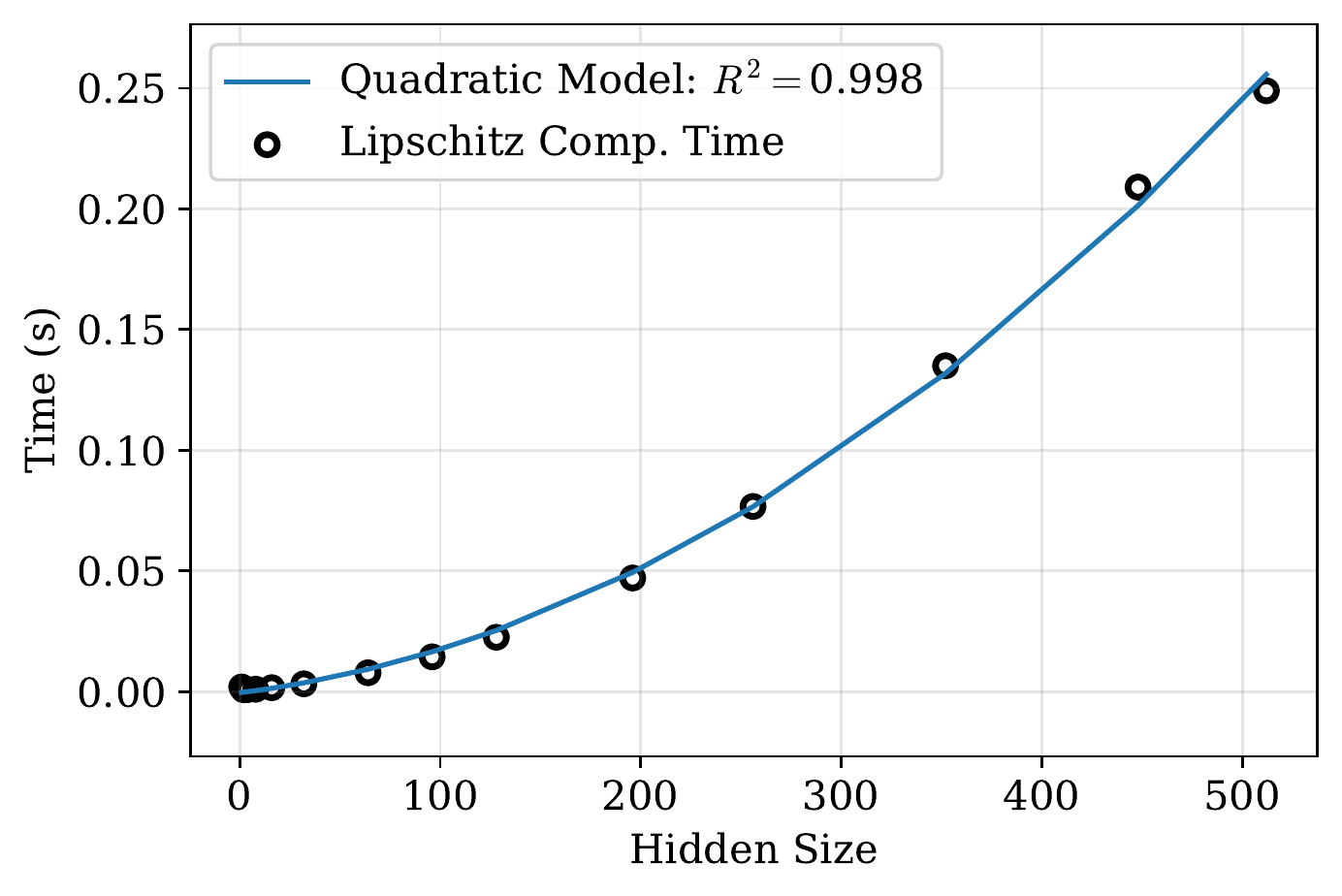}}
\caption{Runtime versus hidden size. The results of fitting a curve with polynomial 2 is shown.}\label{fig:runtime_hiddensize}
\end{subfigure} \hfill
\begin{subfigure}{0.45\linewidth}
\centering
\centerline{\includegraphics[width=\columnwidth]{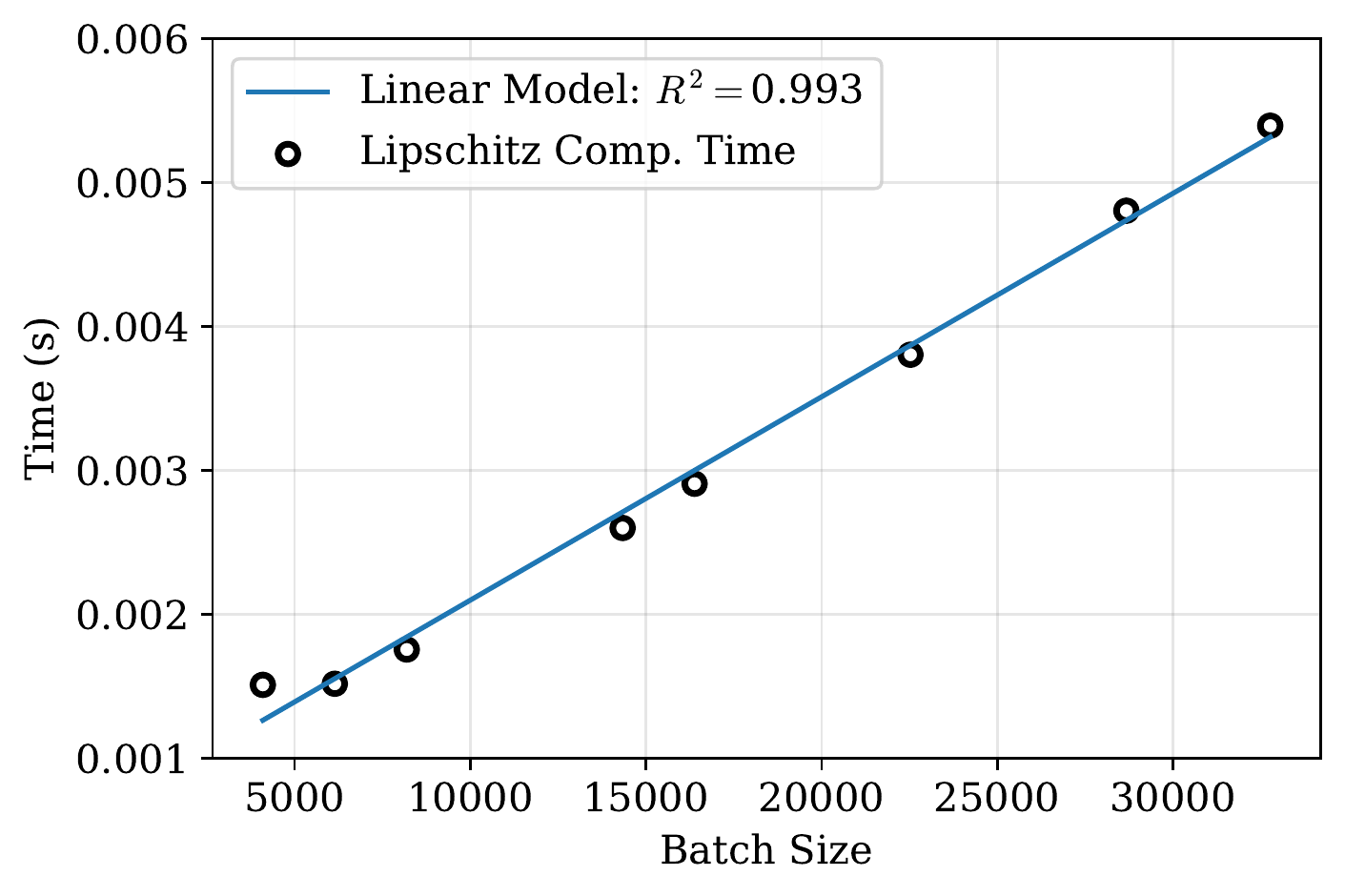}}
\caption{Runtime versus batch size. The results of fitting a linear model is shown.}\label{fig:runtime_batchsize}
\end{subfigure}

\caption{Empirical runtime analysis of Lipschitz computation on a GPU. As a function of hidden size, the runtime is quadratic; as a function of batch size, the runtime is linear. The batch size used for \myref{Figure}{fig:runtime_hiddensize} is $64 \cdot 784$. }\label{fig:lip_runtime}
\end{figure}

\section{Parameter Efficiency for Variationally Dequantized Flows}\label{sec:vflow_comparison}

Most models shown for parameter efficiency have yet to be combined with variational dequantization, thus not allowing for a fair comparison. However, in terms of parameter efficiency, \citet{VFlow2020} performed an ablation study under a fixed parameter budget, specifically with approximately 4 million parameters. The results on a validation set were reported (shown in \myref{Table}{tbl:img_deq_eff}); however, the validation set used by \citet{VFlow2020} was a random subset of 10,000 images from the training set. We performed a similar experiment with approximately four million parameters; however, we use the original test set of CIFAR-10 and report the performance on this set.

\begin{table*}[!tb]
\caption{Comparison of efficiency among variationally dequantized discrete flow-models on CIFAR-10. Flow++ and VFlow results taken from \citet{VFlow2020}. \textit{The validation set used by \citet{VFlow2020} is different from the set we evaluated our model on; specifically,  \citet{VFlow2020} used a random subset of the training set whereas we used CIFAR-10's official test set.}}\label{tbl:img_deq_eff}
\centering
\begin{tabular}{lcc}
\toprule
 & {Bits/dim} & {Param. Count}
\\
\midrule
3-channel Flow++ & $3.23$ & 4.02M \\ 
4-channel VFlow & $3.16$ & 4.03M \\ 
6-channel VFlow & $\mathbf{3.13}$ & 4.01M \\ 
\midrule
% Exact lipschitz ablation & $ $  & 1.9M & $ $ & 1.9M  \\ 
ELF-AR (Ours) & $3.18$ & 4.1M \\ 
% AR Residual Flow & -  & - & $3.27$ & 31.9M  \\ 
% UDA QuAR Flow & -  & - & $3.24$ & 31.9M  \\ 
\bottomrule
\end{tabular}
\end{table*}

\section{ReLU Flow}\label{sec:relu_flow}

To show that residual flows and ELF with ReLU does not learn, we trained a flow on a uniform distribution. In \myref{Figure}{fig:relu_vs_elu_flow}, we can see that both ELF and residual flows with ReLU are unable to learn anything meaningful, even though we gave a hidden size of 2048. On the other hand, simply replacing ReLU with ELU improved performance for the residual flow.

All the flows here are placed in between two affine flows. The only difference between ELF and residual flow is that ELF uses the exact Lipschitz computation whereas residual flow uses spectral norm.

\begin{figure*}[!bt]
% \vskip 0.2in
\begin{center}
\begin{subfigure}{0.5\linewidth}
\centerline{\includegraphics[width=\textwidth]{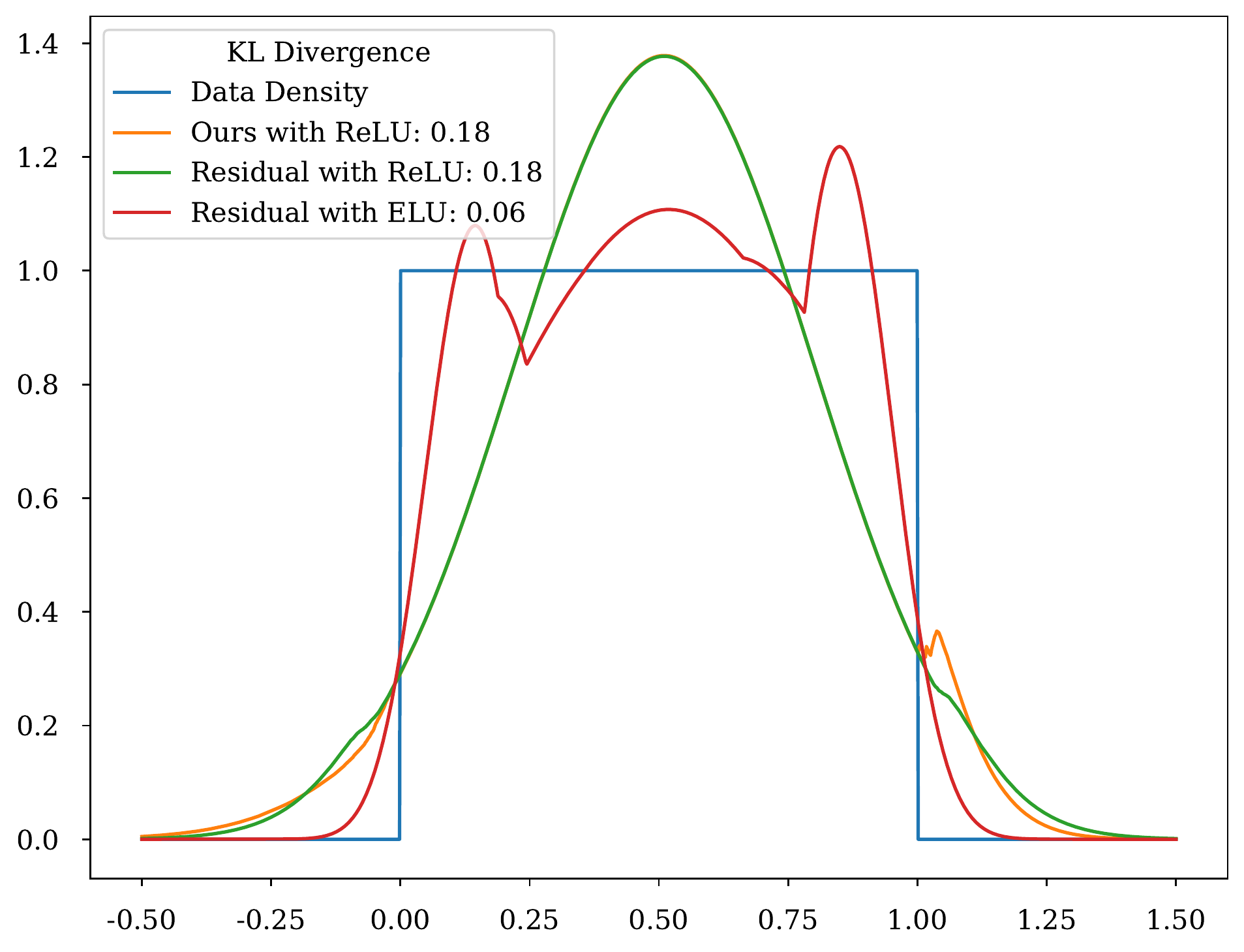}}
\end{subfigure}
\end{center}
\caption{Comparison of using ReLU vs ELU in flows}\label{fig:relu_vs_elu_flow}
% \vskip -0.2in
\end{figure*}

Speculatively, we believe that the fact ReLU does not work comes from a combination of the fact that we optimize flows using gradient based methods, the loss function of flows requires first derivatives (meaning gradient based methods require second derivatives of the flow) and that ReLU has zero second derivative everywhere.

\section{More Synthetic Data Results}\label{sec:synth_checkboard}

\begin{figure}[!tb]
\begin{subfigure}{0.45\linewidth}
\centering
\centerline{\includegraphics[width=\columnwidth]{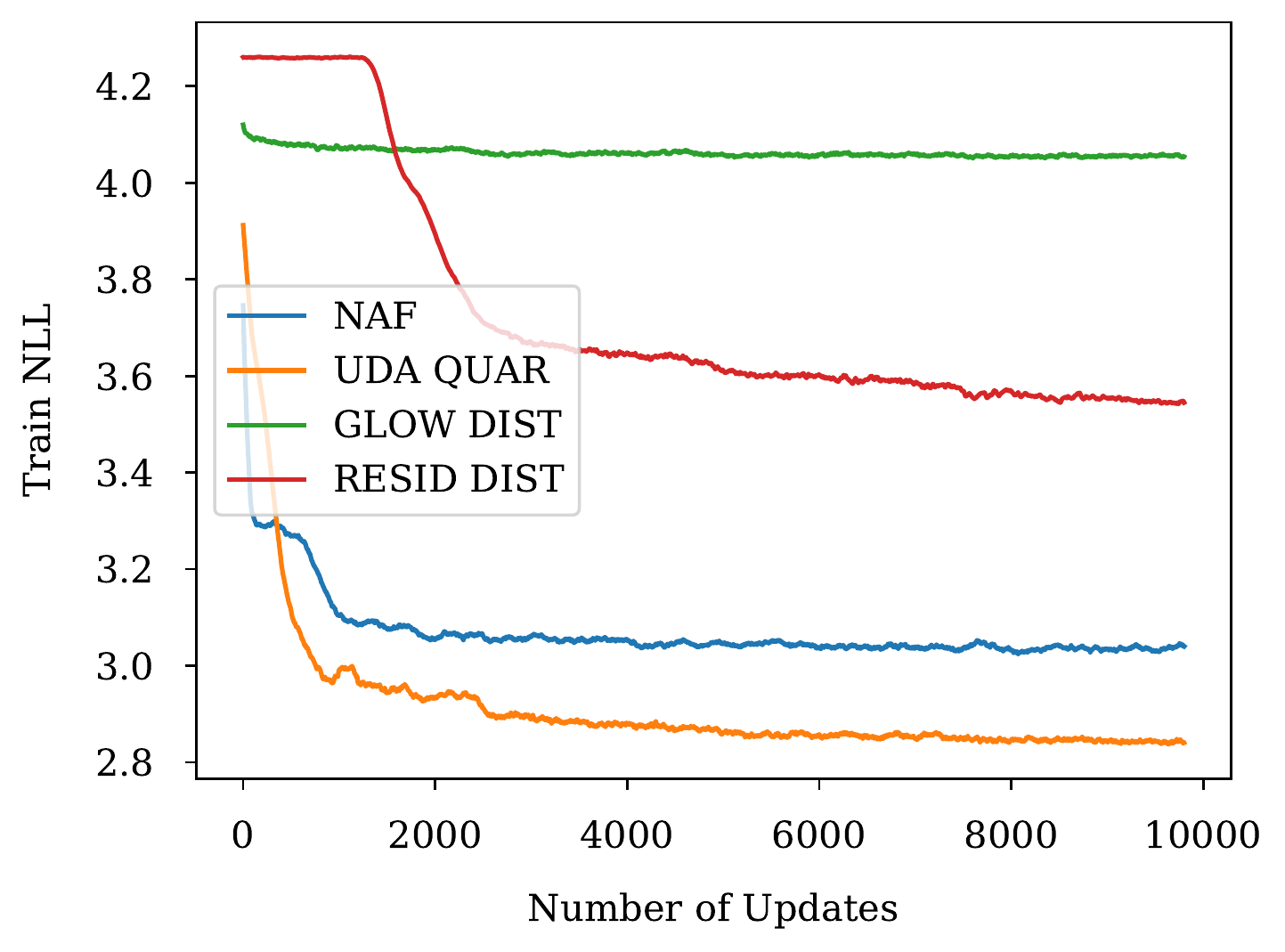}}
\end{subfigure} \hfill
\begin{subfigure}{0.45\linewidth}
\centering
\centerline{\includegraphics[width=\columnwidth]{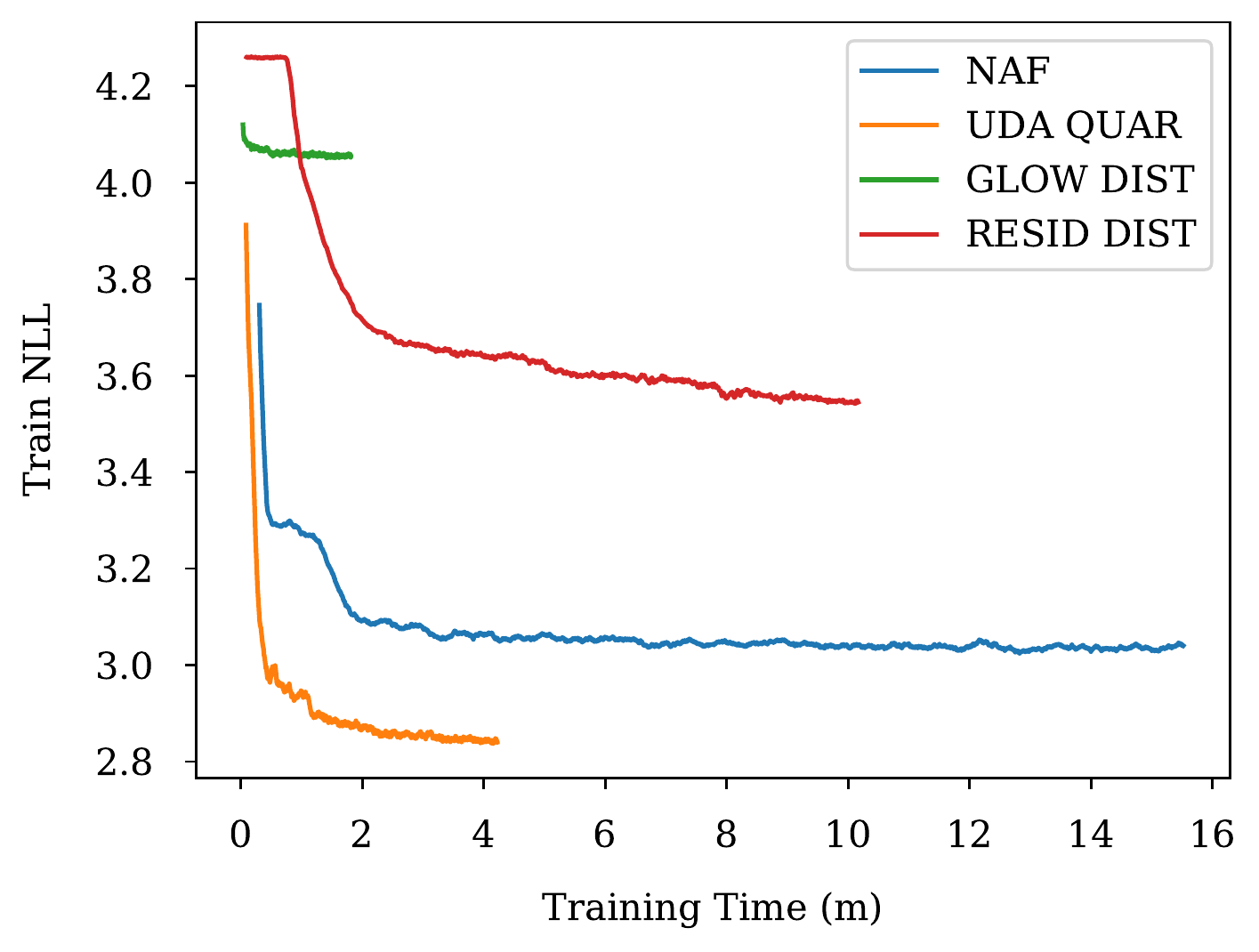}}
\end{subfigure}
\caption{Training loss curves for mixture of eight gaussians}\label{fig:train_eight}
\end{figure}

\begin{figure}[!tb]
\begin{subfigure}{0.45\linewidth}
\centering
\centerline{\includegraphics[width=\columnwidth]{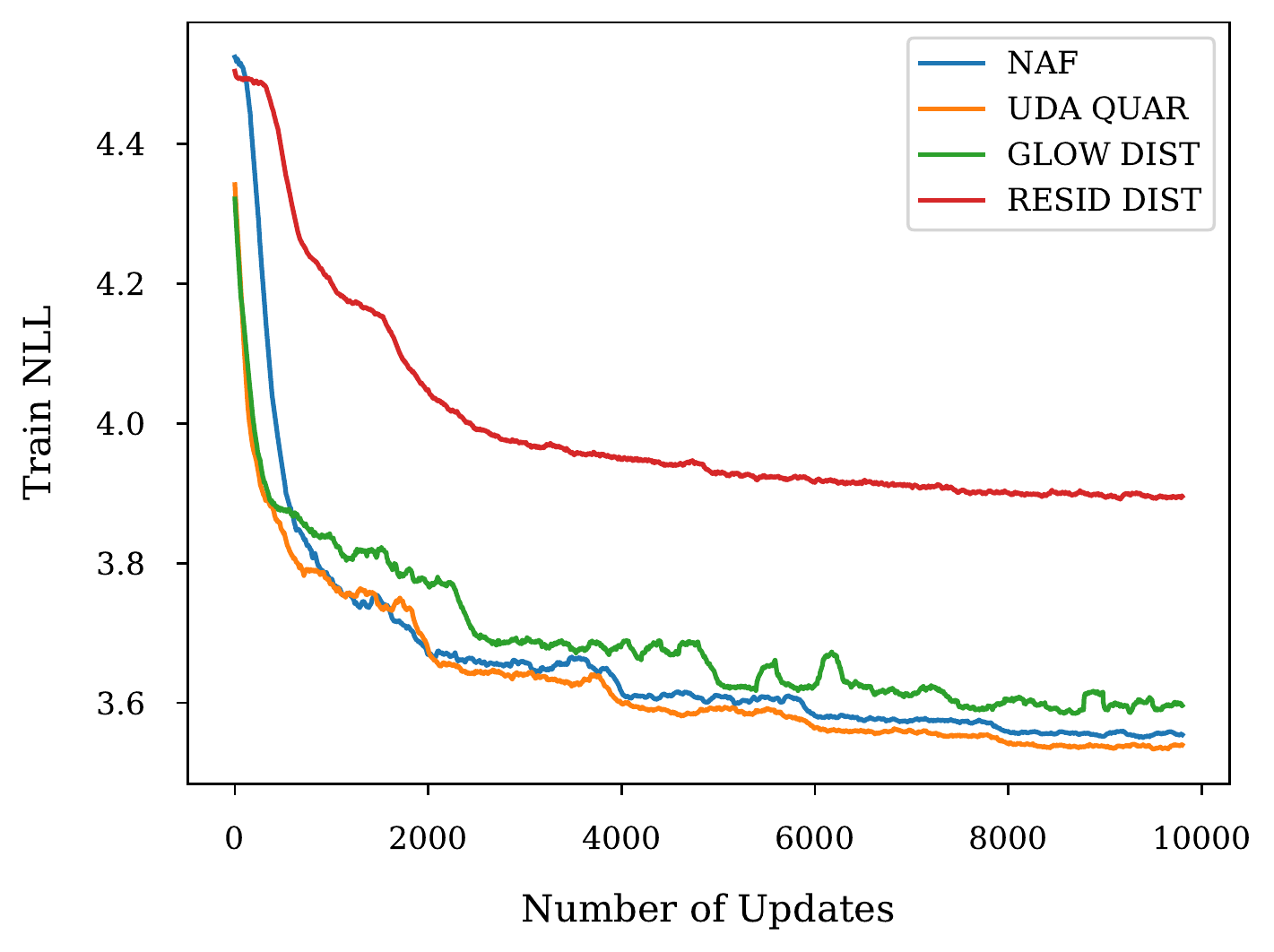}}
\end{subfigure} \hfill
\begin{subfigure}{0.45\linewidth}
\centering
\centerline{\includegraphics[width=\columnwidth]{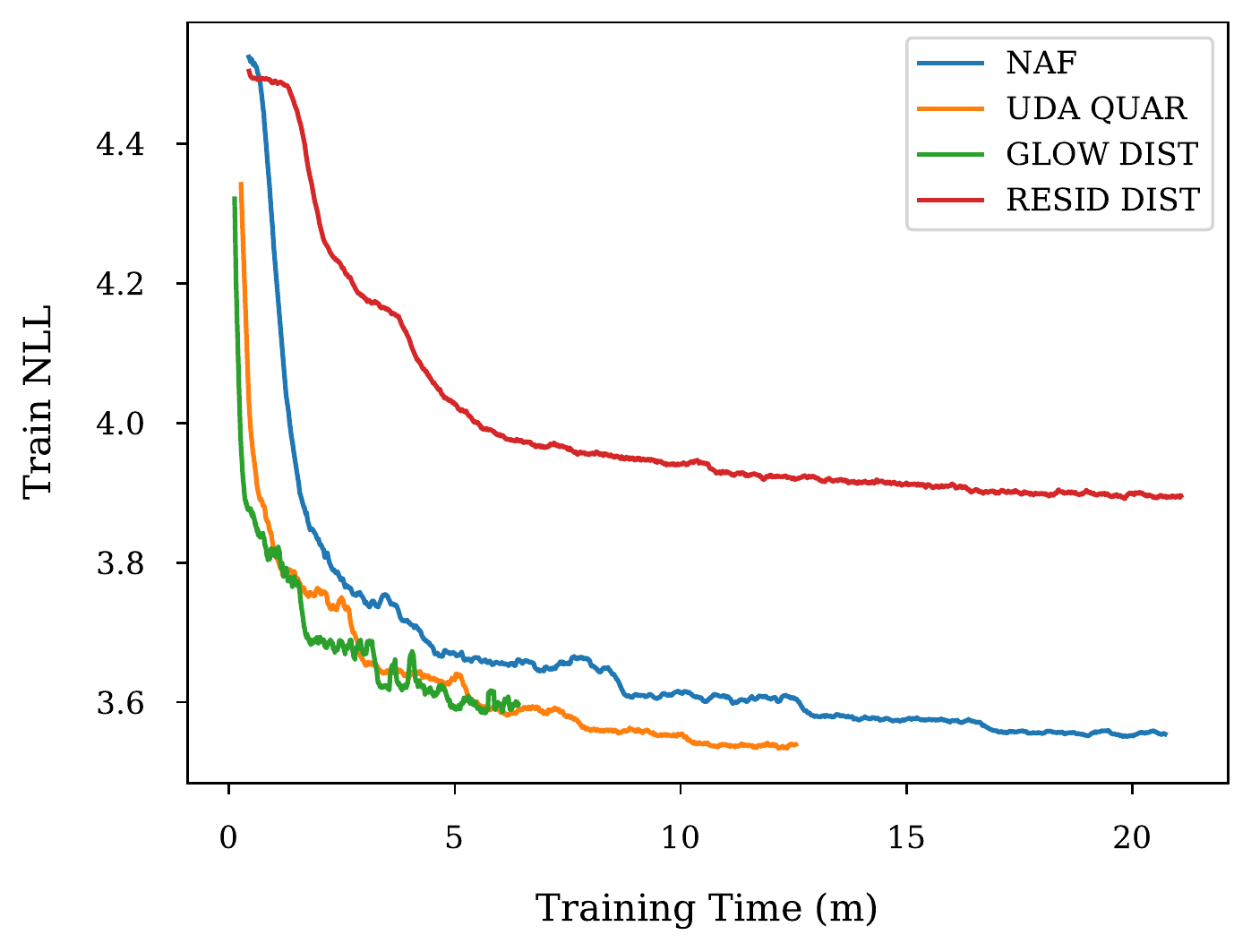}}
\end{subfigure}
\caption{Training loss curves for checkerboard data}\label{fig:train_checker}
\end{figure}

\begin{figure}[!tb]
\begin{subfigure}{0.3\linewidth}
\centering
\centerline{\includegraphics[width=\columnwidth]{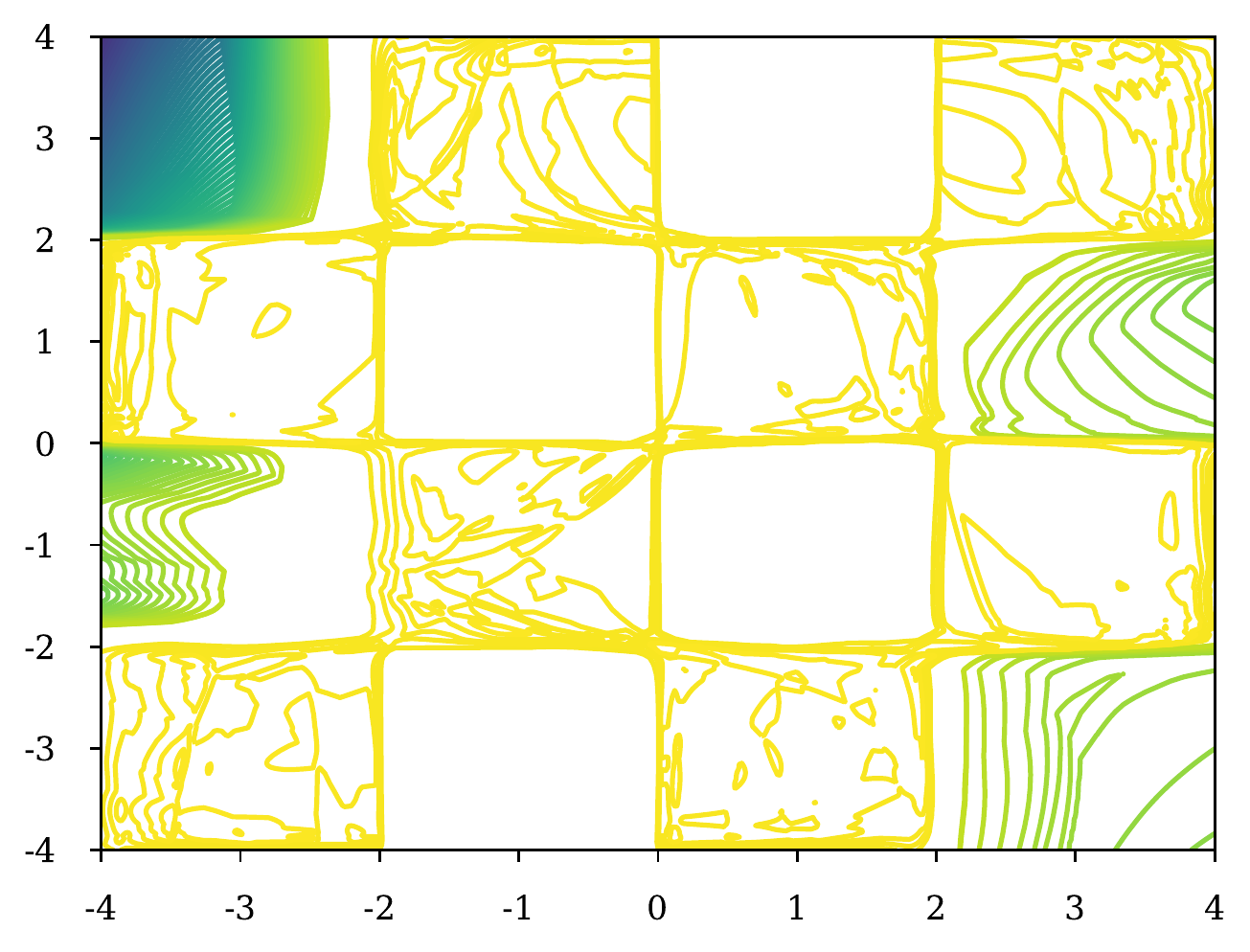}}
\caption{ELF-AR (Ours)}
\end{subfigure} \hfill
\begin{subfigure}{0.3\linewidth}
\centering
\centerline{\includegraphics[width=\columnwidth]{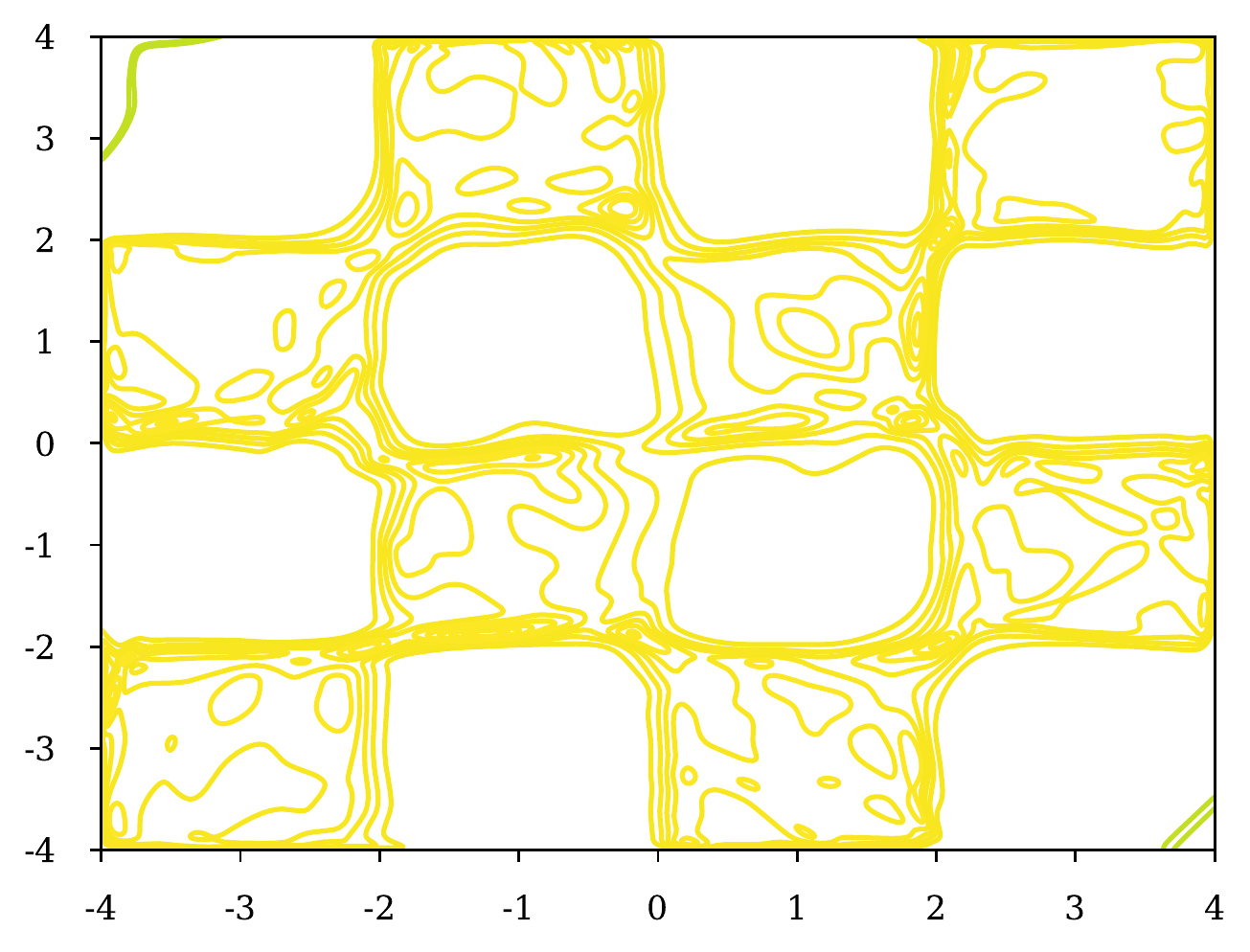}}
\caption{CP-Flow}
\end{subfigure}

\begin{subfigure}{0.3\linewidth}
\centering
\centerline{\includegraphics[width=\columnwidth]{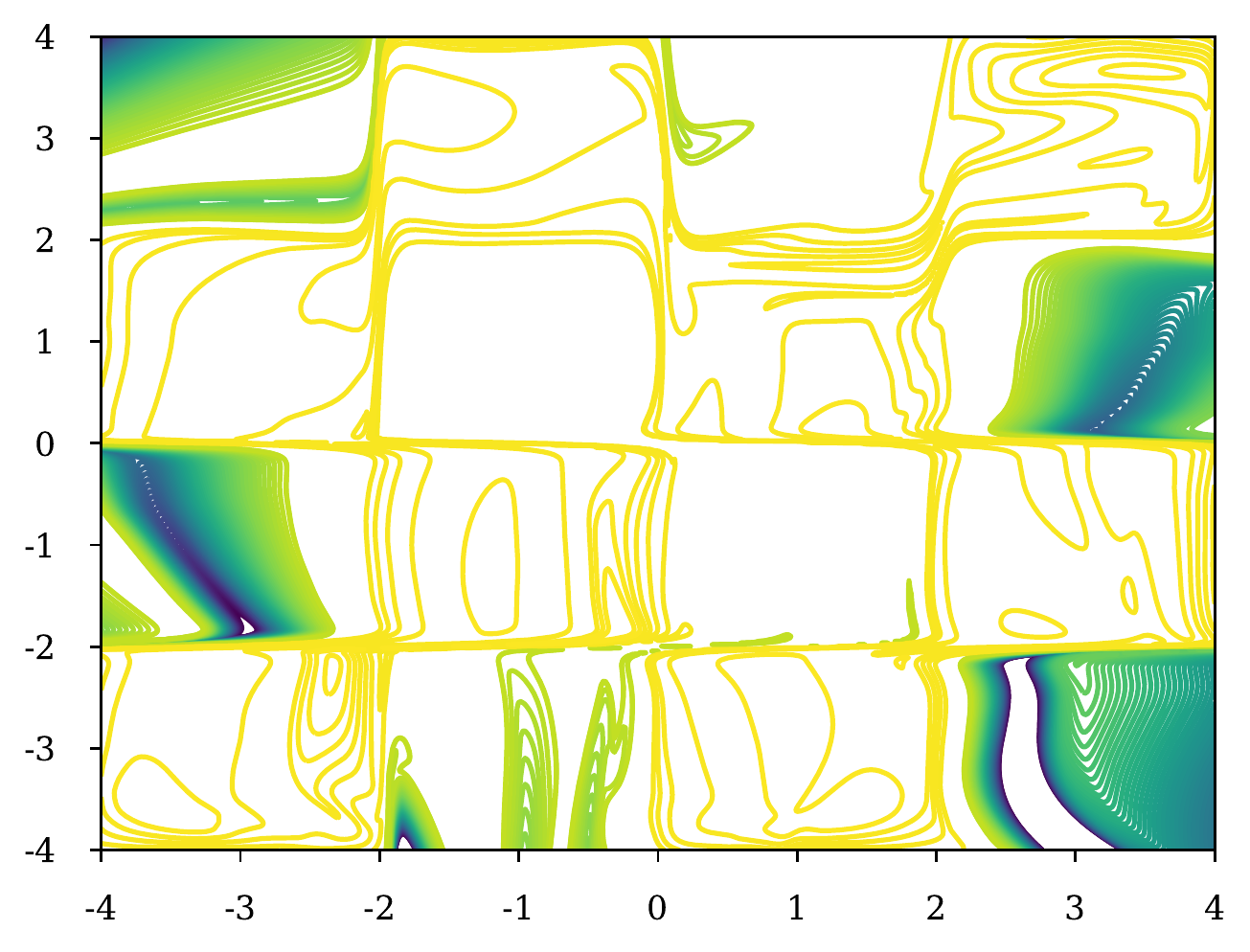}}
\caption{Glow (IAF)}
\end{subfigure} \hfill
\begin{subfigure}{0.3\linewidth}
\centering
\centerline{\includegraphics[width=\columnwidth]{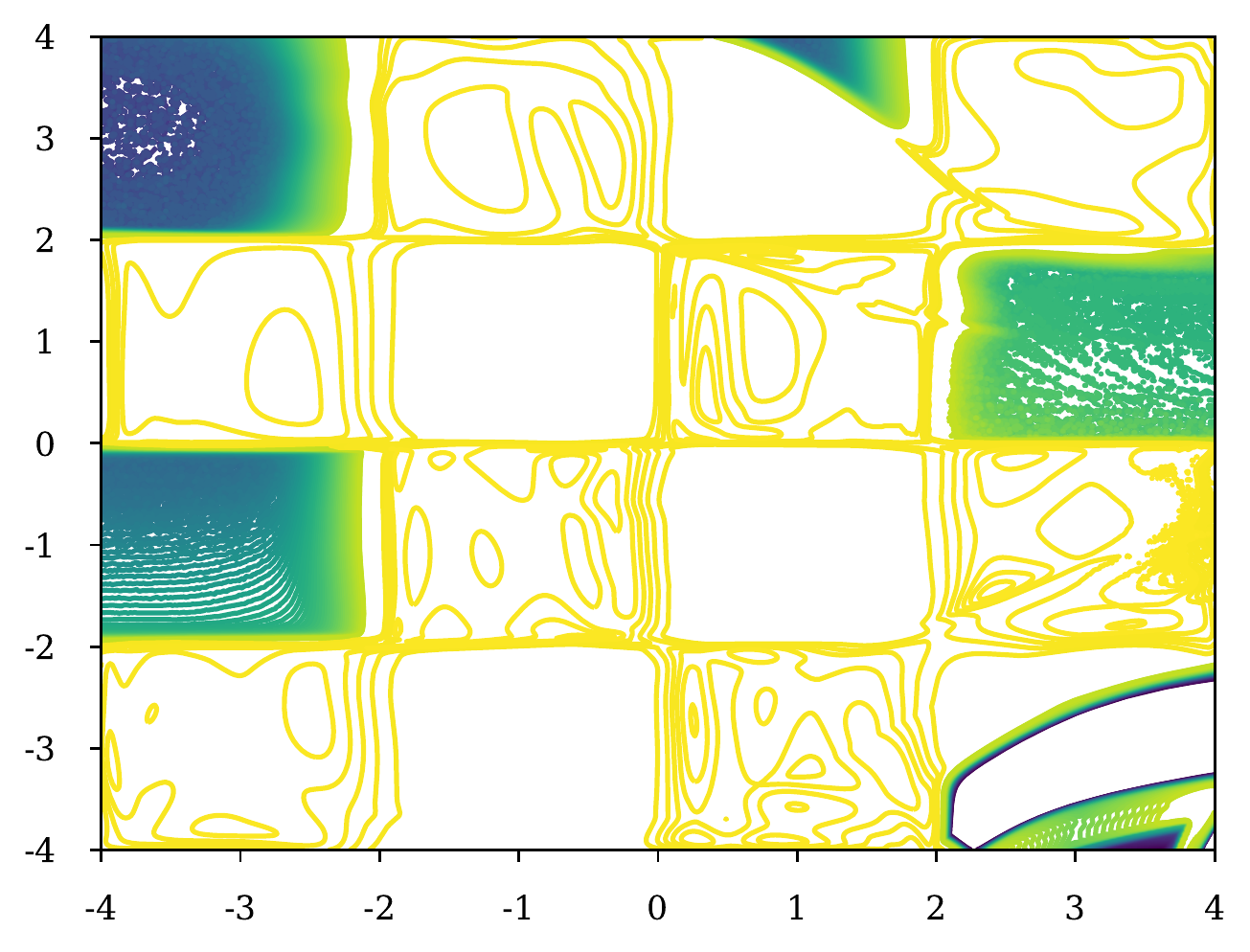}}
\caption{NAF}
\end{subfigure} \hfill
\begin{subfigure}{0.3\linewidth}
\centering
\centerline{\includegraphics[width=\columnwidth]{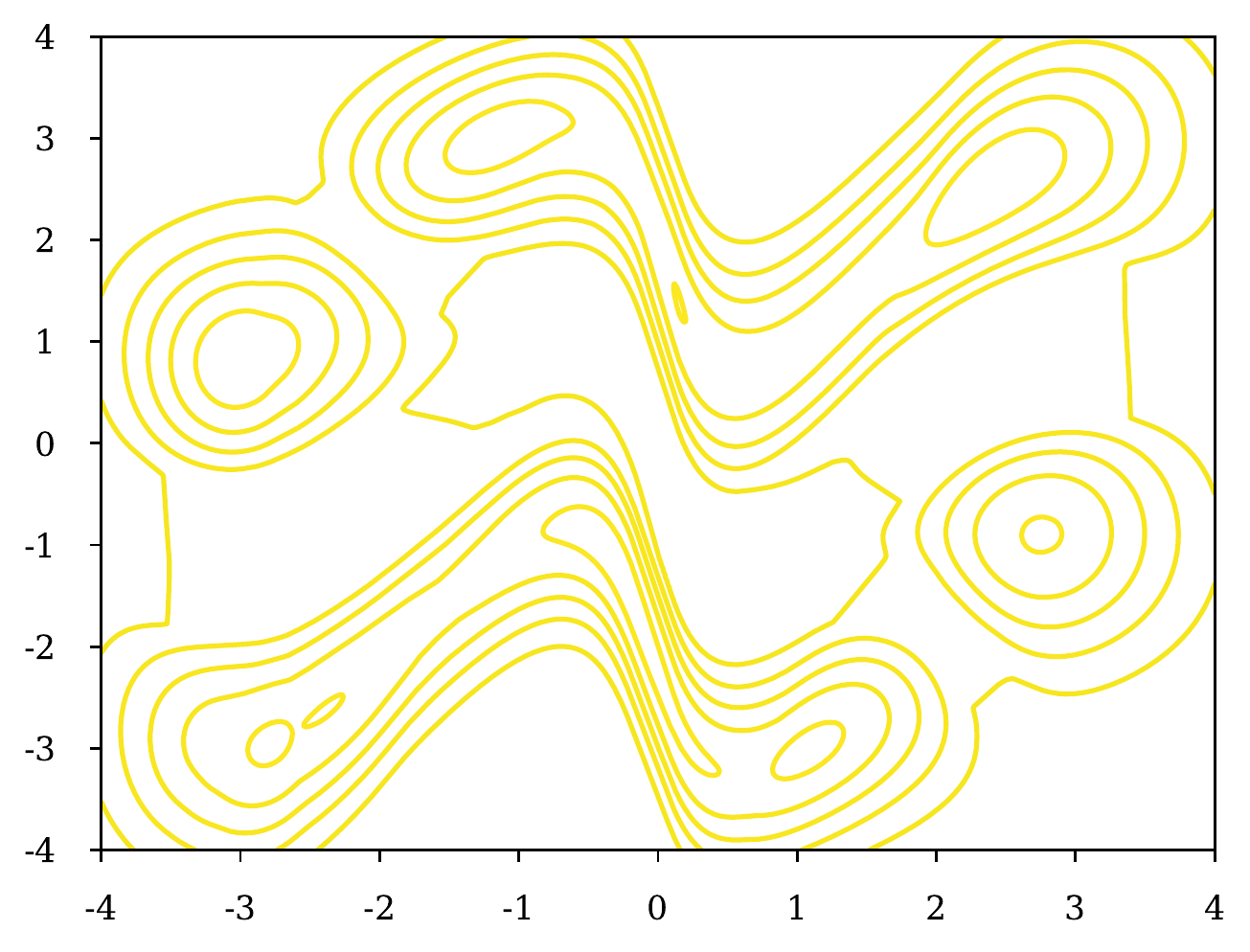}}
\caption{Residual Flow}
\end{subfigure}
\caption{The contour plots in log space of mixture of Gaussians. The levels shown in each subfigure are the same.}\label{fig:synth_checkerboard}
\end{figure}

\begin{table*}[!bt]
\caption{Comparison of different flows on fitting to checkerboard data where the architectures for each were chosen so that they had the same number of layers and approximately equivalent number of parameters (500K). $^\dagger$Implementation taken from \href{https://github.com/CW-Huang/CP-Flow}{https://github.com/CW-Huang/CP-Flow}.}\label{tbl:synth_checkerboard}
% \hspace*{-1.0cm}
\centering
\begin{tabular}{lcccr}
\toprule
     & \multicolumn{4}{c}{Checkerboard}
\\
 \cmidrule(lr){2-5} 
 & {Log-Likelihood} & {Train Time (m)}& {Inference Time (s)} & {Sample Speed (s)}  
\\
\midrule
Affine Coupling &  -3.60 & 6.4 & 0.04 & 0.004 \\ 
NAF & \textbf{-3.56} & 20.7 & 0.06 & \na \\ 
Residual Flow & -3.90 & 21.1 & 0.10 & 0.31  \\ 
CP-Flow & -3.60 & 143 & 0.69 & 3.5  \\ 
ELF-AR (Ours)  & \textbf{-3.56} & 12.6 & 0.20 & 0.54 \\ 
\bottomrule
\end{tabular}
\end{table*}

Similar to the experiment in \myref{Section}{sec:synthetic_data}, we tested out numerous flows on checkerboard data (\myref{Figure}{fig:synth_checkerboard}). However, one difference is that we stacked five flows instead of just one. In \myref{Table}{tbl:synth_checkerboard}, we see that ELF-AR and NAF are both able to perform strongly; however, a key difference being that we can sample from ELF-AR.

We further show the training NLL as a function of time and updates in \myref{Figure}{fig:train_eight} and \myref{Figure}{fig:train_checker}. We do not show CP-Flow as the loss used during training is a surrogate loss.

\section{Experimental Setup}\label{sec:architecture}

For all experiments, we trained on a single Tesla V100-SXM2-32GB GPU, except for the Imagenet experiments where we used four GPUs for Imagenet 32 and eight GPUs for Imagenet 64. All image experiments were run for up to 2 weeks.

\subsection{Synthetic Data}\label{sec:synth_architecture}

The mixture of Eight Gaussians was created using \href{https://github.com/CW-Huang/CP-Flow}{https://github.com/CW-Huang/CP-Flow}. The checkerboard dataset was created using \href{https://github.com/rtqichen/residual-flows}{https://github.com/rtqichen/residual-flows}.

For each method in our synthetic data, we used four hidden layers within each flow. For the results in \myref{Table}{tbl:synth_data}, we used one flow; for the results in \myref{Table}{tbl:synth_checkerboard}, we used five flows.

For mixture of eight gaussians, we used a hidden size of 256 for Affine Coupling and Residual Flow, 192 for ELF-AR, 384 for CP-Flow, and 256 for NAFs. For checkerboard data, we used a hidden size of 196 for Affine Coupling and Residual Flow, 128 for ELF-AR, 256 for CP-Flow, and 160 for NAFs. 

For each method, we optimized using Adam \citep{Adam2015}, starting with a learning rate of 2e-3 or 5e-3 (depending on whichever gave stronger performance). We trained for 10K steps with a batch size of 128, halving the learning rate every 2.5K steps.

\subsection{Tabular Data}\label{sec:tabular_architecture}

For all five datasets, we used five flows with five hidden layers. We hyperparameter tuned using a hidden size of $8d$, $16d$, $32d$ or $64d$ where $d$ is the dimensionality of the dataset. We further explored a weight decay of 1e-4, 1e-5, and 1e-6. For each dataset, we trained for up to 1000 epochs, early stopping on the validation set. We used a batch size of 1024 for every dataset except MINIBOONE where we used a batch size of 128. 

We optimized using Adam \citep{Adam2015} starting with a learning rate of 1e-3.
Further, we clipped the norm of gradients to the range $[-1, 1]$ and reduced the learning rate by half (to no lower than 1e-4) whenever the validation loss did not improve by more than 1e-3, using patience of five epochs.

Of the five datasets, MINIBOONE and BSDS300 eventually started to overfit whereas the other three did not. We expect larger parameterizations with stronger regularization might further improve performance on these two datasets.

\subsection{Image Data}\label{sec:image_architecture}

For the image models, we used a multiscale architecture. For all four datasets, we used three scales, four flows per scale for MNIST and CIFAR-10 and eight flows per scale for the two Imagenet datasets. We used a hidden size of 256 for ELF. For the hypernetwork, we used the convolutional residual blocks from \cite{Oord2016PixelRNN}; the residual block is:
\begin{equation}\label{eqn:pixelcnn_resid}
 \text{ReLU} \rightarrow \text{3x3x2HxH Conv} \rightarrow \text{ReLU} \rightarrow \text{1x1xHxH Conv} \rightarrow \text{ReLU} \rightarrow \text{3x3xHx2H Conv} 
 \end{equation}
where the notation used for the convolution is the first two are the kernel sizes in the height and width direction, and the last two numbers are the input and output size respectively. We used $H=192$ for all our full scale image experiments. Further, we used five residual blocks per flow.

For variational dequantization, we used a similar architecture for each flow except we condition on the image we are dequantizing. We used only one scale for dequantization instead of multiscale. To condition on an image, we first passed it through 4 residual blocks with hidden size 32:
$$ \text{ReLU} \rightarrow \text{3x3x32x32 Conv} \rightarrow \text{ReLU} \rightarrow \text{3x3x32x32  Conv} $$
To condition each flow on this representation, we concatenated it to the beginning of each residual blocks (\myref{Equation}{eqn:pixelcnn_resid}).

For training, we optimized using Adam \citep{Adam2015} with a learning rate of 1e-3 with a batch size of 64. We clipped the norm of gradients to the range $[-1, 1]$. We warmed up the learning rate over 100 steps; after 100K updates, we start to exponentially decay the learning rate by 0.99999 at every step until a learning rate of 2e-4, similar to the schedule used by \cite{VFlow2020}.

For our small architectures (\myref{Section}{sec:image_param_eff}), we used three residual blocks in each flow with a hidden size of 80 for CIFAR-10 and 84 for MNIST. For the experiment in \myref{Section}{sec:vflow_comparison}, we used a hidden size of 88 and the same dequantization model as was used for our full image model.

To handle the range of the data being $[0, 1]$, we used a logit transformation as the first layer in the network \citep{dinh2014nice}. 
However, for MNIST, we went one step further. Since majority of the pixels values are 0, before the logit transformation, we first transformed the data using the CDF of a mixture of $\text{Uniform}(0, 1/256)$ and $\text{Uniform}(1/256, 1)$, weighed by the probability of a pixel value being in the corresponding range in the training set. This trick helps to make the pixel values more uniform. We found this improved the ease of learning.

For all image models, we used Polyak averaging \citep{Polyak1992} with decay 0.999.

\section{Bits Per Dimension}\label{sec:bpd}

The performance of log-likelihood models for images is often defined using bits per dimension. Given a dequantization distribution $q(x)$ for $x \in \RR^d$, the bits per dimension is defined as
$$ \frac{\log p(x) - \log q(x)}{d \log 2} $$
\section{Image Generations}\label{sec:img_gen}

We generated samples from our model with the best bits per dimensions for MNIST (\myref{Figure}{fig:mnist_generations}) and CIFAR-10 (\myref{Figure}{fig:cifar_generations}).
\begin{figure*}[ht]
% \vskip 0.2in
\begin{center}
\begin{subfigure}{0.9\linewidth}
\centerline{\includegraphics[width=\textwidth]{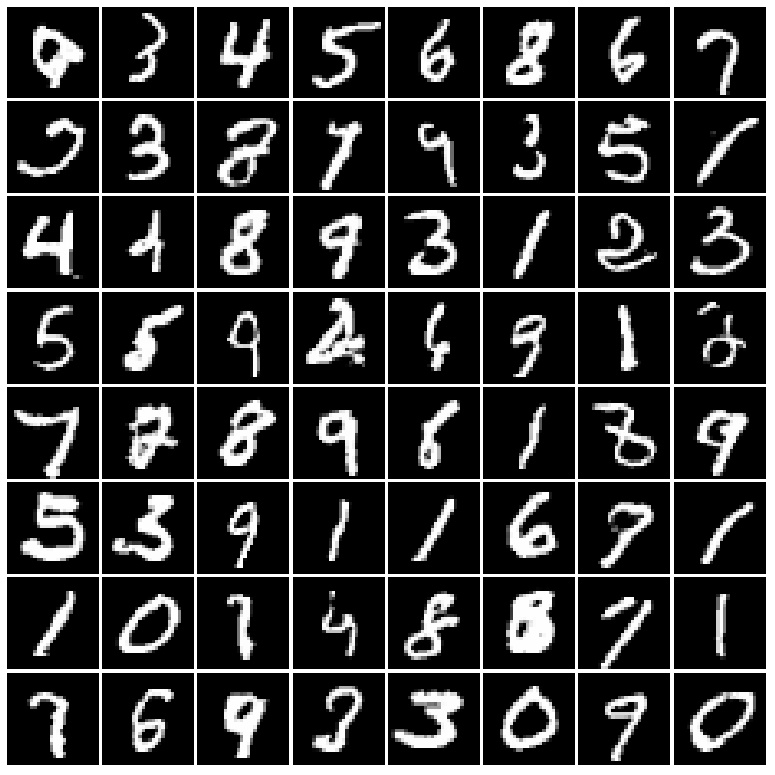}}
\end{subfigure}
\caption{Random samples from MNIST}
\label{fig:mnist_generations}
\end{center}
% \vskip -0.2in
\end{figure*}
\begin{figure*}[ht]
% \vskip 0.2in
\begin{center}
\begin{subfigure}{0.9\linewidth}
\centerline{\includegraphics[width=\textwidth]{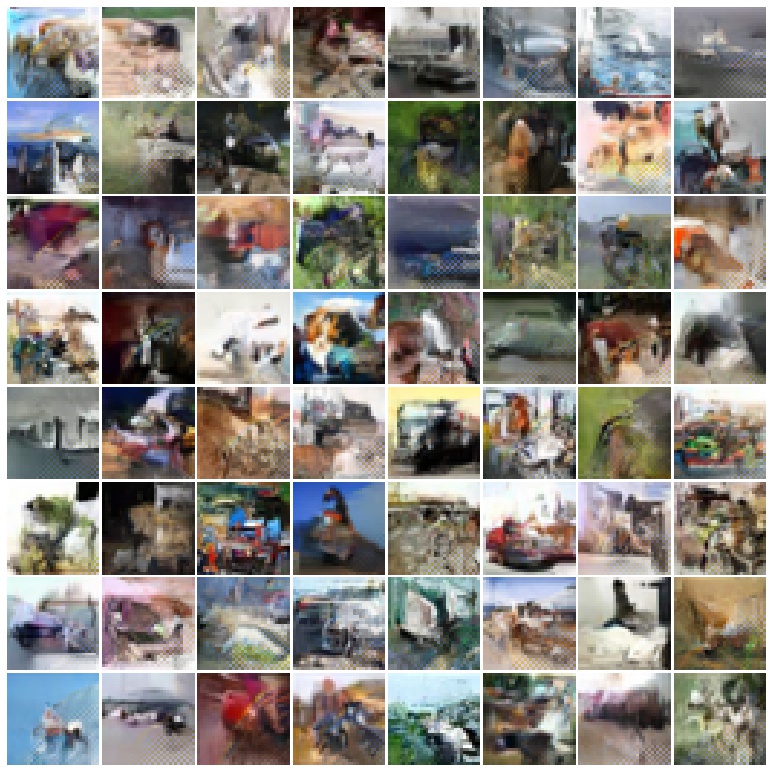}}
\end{subfigure}
\caption{Random samples from CIFAR-10}
\label{fig:cifar_generations}
\end{center}
% \vskip -0.2in
\end{figure*}

\section{Lipschitz Computation Code}\label{sec:cuda_code}

In this section, we give our custom CUDA code for Lipschitz computation.

\newpage
\begin{minted}{cpp}
/* 
    Inspired by https://pytorch.org/tutorials/advanced/cpp_extension.html 
    Implements our algorithm in CUDA/C++.
    Copyright 2021 Achintya Gopal. Distributed under the terms of the Apache 2.0 license.
 */
#include <ATen/ATen.h>
#include <torch/extension.h>

#include <cuda.h>
#include <cuda_runtime.h>

#include <vector>
#include <stdio.h> 
#include <math.h>

// DIFFERENT CONSANTS.
// OPTIMIZING THESE MIGHT IMROVE PEFROMANCE FURTHER. 
#define TILE_WIDTH_1 4

__global__
void cuda_elf_lip_op(
    at::PackedTensorAccessor32<float, 3, torch::RestrictPtrTraits> block_diag1,
    at::PackedTensorAccessor32<float, 3, torch::RestrictPtrTraits> block_diag2,
    at::PackedTensorAccessor32<float, 3, torch::RestrictPtrTraits> test_vals,
    at::PackedTensorAccessor32<float, 3, torch::RestrictPtrTraits> test_vals1,
    at::PackedTensorAccessor32<float, 3, torch::RestrictPtrTraits> test_vals2,
    at::PackedTensorAccessor32<float, 3, torch::RestrictPtrTraits> res,
    int B, int I, int S
) {
    int b = blockIdx.x * blockDim.x + threadIdx.x;
    int i = blockIdx.y * blockDim.y + threadIdx.y;
    int s = blockIdx.z * blockDim.z + threadIdx.z;

    if ((b < B) && (i < I) && (s < (2 * S + 1 ))) {
        float sum_res = 0.0;
        float test_val = test_vals[b][i][s];

        int diff1_pos;
        int diff1_neg;
        int diff2_pos;
        int diff2_neg;
        int w1_pos;
        int w1_neg;
        float bd_prod;
        float denom;
        float diff2;

        for (int h = 0; h < S; h++) {

            diff1_pos = (test_val > test_vals1[b][i][h]);
            diff1_neg = 1 - diff1_pos;

            diff2_pos = (test_val > test_vals2[b][i][h]);
            diff2_neg = 1 - diff2_pos;

            w1_pos = (block_diag1[b][i][h] > 0);
            w1_neg = 1 - w1_pos;
            bd_prod = block_diag1[b][i][h] * block_diag2[b][i][h];

            diff2 = test_val - test_vals2[b][i][h];
            denom = test_vals2[b][i][h] - test_vals1[b][i][h];
            if (denom == 0) {
                denom = 1;
            }

            sum_res += (
                diff1_pos * diff2_pos * w1_pos + 
                diff1_neg * diff2_neg * w1_neg + (
                (diff1_pos * diff2_neg - diff1_neg * diff2_pos)
                    * (w1_pos - w1_neg)
                    * diff2
                    / denom
                )
            ) * bd_prod;
        }
        res[b][i][s]= sum_res;
    }
}


void cuda_elf_lip(
    at::Tensor block_diag1,  // B x I x S
    at::Tensor block_diag2,  // B x I x S
    at::Tensor test_vals,  // B x I x (2 * S + 1)
    at::Tensor test_vals1,  // B x I x S
    at::Tensor test_vals2,  // B x I x S
    at::Tensor res  // B x I x (2 * S + 1)
) {
    int B = test_vals1.size(0);
    int I = test_vals1.size(1);
    int S = test_vals1.size(2);
    int H = (2 * S + 1);

    auto block_diag1A = block_diag1.packed_accessor32<float, 3, at::RestrictPtrTraits>();
    auto block_diag2A = block_diag2.packed_accessor32<float, 3, at::RestrictPtrTraits>();
    auto test_valsA = test_vals.packed_accessor32<float, 3, at::RestrictPtrTraits>();
    auto test_vals1A = test_vals1.packed_accessor32<float, 3, at::RestrictPtrTraits>();
    auto test_vals2A = test_vals2.packed_accessor32<float, 3, at::RestrictPtrTraits>();
    auto resA = res.packed_accessor32<float, 3, at::RestrictPtrTraits>();

    int TILE_WIDTH_2 = 1;
    int TILE_WIDTH_3 = 1;
    if (H > I) {
        TILE_WIDTH_2 = 1;
        TILE_WIDTH_3= 32;
    } else {
        TILE_WIDTH_2 = 32;
        TILE_WIDTH_3= 1;
    }

    unsigned int grid_x = (B + TILE_WIDTH_1 - 1) / TILE_WIDTH_1;
    unsigned int grid_y = (I + TILE_WIDTH_2 - 1) / TILE_WIDTH_2;
    unsigned int grid_z = (H + TILE_WIDTH_3 - 1) / TILE_WIDTH_3;

    dim3 grid_1(grid_x, grid_y, grid_z);
    dim3 block_1(TILE_WIDTH_1, TILE_WIDTH_2, TILE_WIDTH_3);

    cuda_elf_lip_op << < grid_1, block_1 >> > (
        block_diag1A,
        block_diag2A,
        test_valsA,
        test_vals1A,
        test_vals2A,
        resA, B, I, S
    ); 
}
\end{minted}

\end{document}